\documentclass[sigconf,anonymous=false]{acmart}

\usepackage{bm}
\usepackage{amsmath}
\usepackage{amsthm}
\newtheorem{theorem}{Theorem}[section]
\newtheorem{lemma}{Lemma}[section]

\usepackage{graphicx}
\usepackage{float}
\usepackage{amsfonts}

\usepackage{url}
\usepackage{booktabs}
\usepackage{microtype}      
\usepackage{xcolor}         
\usepackage[utf8]{inputenc}
\usepackage{caption}
\usepackage{subcaption}
\usepackage{multirow}
\usepackage{textcomp}
\usepackage[T1]{fontenc}
\usepackage{balance} 

\usepackage[]{algorithm2e}
\RestyleAlgo{ruled}

\newcommand{\vecX}{\mathbf{x}}
\newcommand{\vecW}{\mathbf{w}}
\newcommand{\tS}{h} 

\AtBeginDocument{%
  \providecommand\BibTeX{{%
    \normalfont B\kern-0.5em{\scshape i\kern-0.25em b}\kern-0.8em\TeX}}}

\copyrightyear{2022}
\acmYear{2022}
\setcopyright{acmcopyright}
\acmConference[MM '22] {Proceedings of the 30th ACM International Conference on Multimedia }{October 10--14, 2022}{Lisboa, Portugal.}
\acmBooktitle{Proceedings of the 30th ACM International Conference on Multimedia (MM '22), October 10--14, 2022, Lisboa, Portugal}
\acmPrice{15.00}
\acmISBN{978-1-4503-9203-7/22/10}
\acmDOI{10.1145/3503161.3548182}

\begin{document}

\title{InDiD: Instant Disorder Detection via a Principled Neural Network}

\author{Evgenia Romanenkova}
\affiliation{%
  \institution{Skolkovo Institute of Science and Technology (Skoltech)}
  \city{Moscow}
  \country{Russia}
}
\email{shulgina@phystech.edu}

\author{Alexander Stepikin}
\affiliation{%
  \institution{Skolkovo Institute of Science and Technology (Skoltech),}
  \institution{Moscow Institute of Physics and Technology}
  \city{Moscow}
  \country{Russia}
}
\email{stepikin.al@phystech.edu}

\author{Matvey Morozov}
\affiliation{%
  \institution{Skolkovo Institute of Science and Technology (Skoltech)}
  \city{Moscow}
  \country{Russia}
}
\email{m.morozov@skoltech.ru}

\author{Alexey Zaytsev}
\affiliation{%
  \institution{Skolkovo Institute of Science and Technology (Skoltech)}
  \city{Moscow}
  \country{Russia}
}
\email{a.zaytsev@skoltech.ru}

\renewcommand{\shortauthors}{Evgenia Romanenkova, Alexander Stepikin, Matvey Morozov, \& Alexey Zaytsev}

\begin{abstract}
For sequential data, a change point is a moment of abrupt regime switch in data streams.
Such changes appear in different scenarios, including simpler data from sensors and more challenging video surveillance data. We need to detect disorders as fast as possible. 
Classic approaches for change point detection (CPD) might underperform for semi-structured sequential data because they cannot process its structure without a proper representation.
We propose a principled loss function that balances change detection delay and time to a false alarm. It approximates classic rigorous solutions but is differentiable and allows representation learning for deep models. 
We consider synthetic sequences, real-world data sensors and videos with change points. We carefully labelled available video data with change point moments and released it for the first time.
Experiments suggest that complex data require meaningful representations tailored for the specificity of the CPD task --- and our approach provides them outperforming considered baselines.
For example, for explosion detection in video, the F1 score for our method is $0.53$ compared to baseline scores of $0.31$ and $0.35$.
\end{abstract}



\maketitle

\section{Introduction}
\label{sec:introduction}

Modern industry uses complicated systems continuously working online that are vital for the well-being of large companies and humankind in general. Systems' collapses or prolonged unavailability lead to significant losses to business owners. Thus, it is essential to detect deviations or, in other words, change points in a system's behaviour as fast as possible.
Typically, data for detection come from a sequential stream represented as either multivariate vectors from sensors or, so-called, semi-structured data such as videos/image sequences \cite{van2020evaluation,sultani2018real}. 
A natural impulse is to use the available historical information to create a change point detector --- a model that warns about distribution disorders online. Such a model can identify human activity change based on smart devices or detect car accidents from video streams.

\begin{figure*}[!ht]
    \centering
    \includegraphics[width=0.5\textwidth]{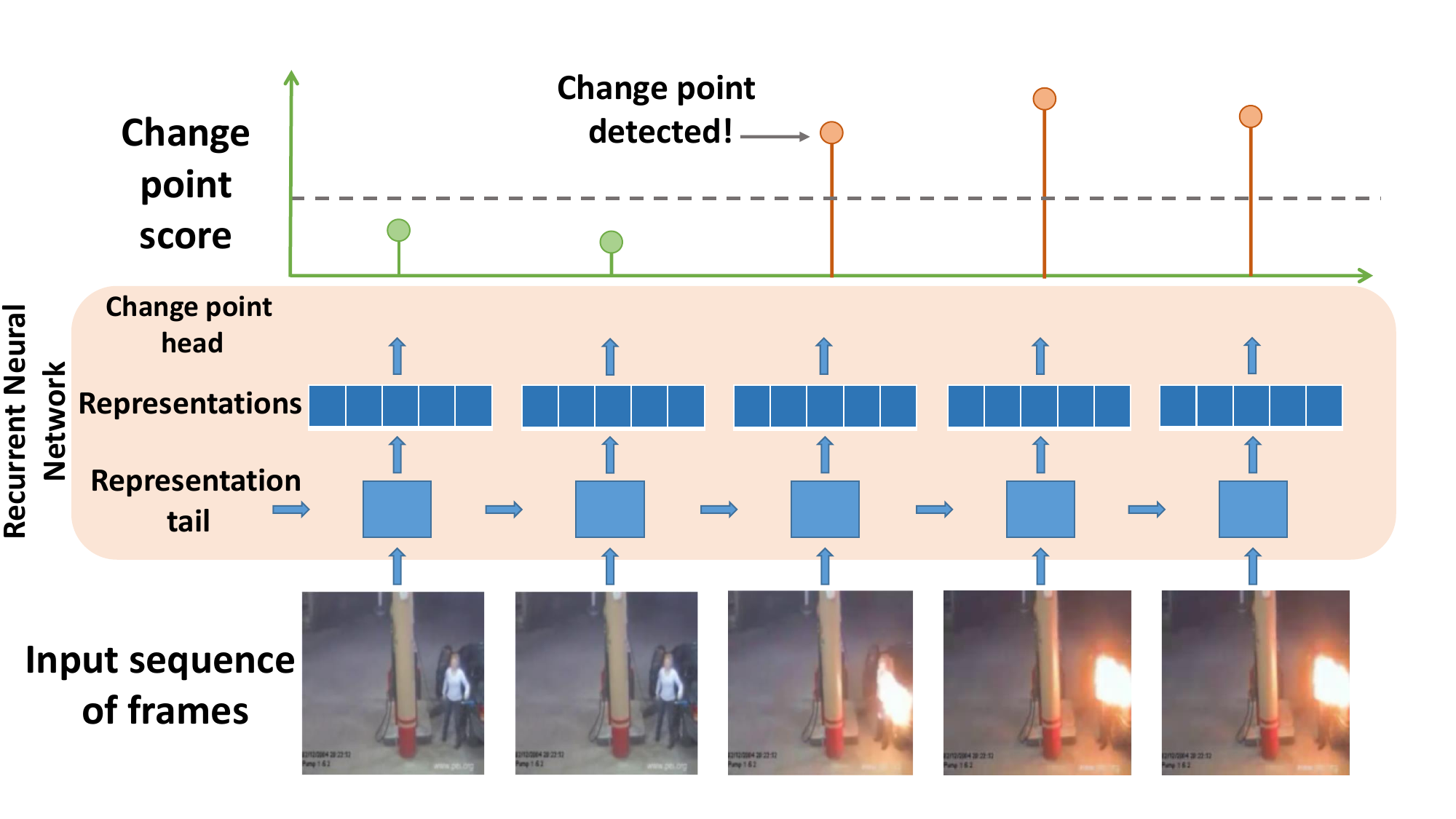}
    \caption{A scheme of work for representation-based change point detection approach}
    \label{fig:cpd_scheme}
\end{figure*}

We can also describe disorders, ruptures, shifts and other changes in sequential data as point anomalies. Currently, the anomaly detection area concentrates on applying machine learning and deep learning techniques. It adopts widespread methods and considers more specific models and loss functions~\cite{sultani2018real}.
However, instead of an accurate \emph{anomaly moment} detection, the majority of existing methods concentrate on classifying whole sequences whether they contain anomalies~\cite{shen2020timeseries,li2018anomaly}.
An alternative approach comes from Change Point Detection (CPD) area: we aim to minimize the delay of disorder detection and the number of false alarms. 
The CPD problem statement better reflects industrial needs in various monitoring systems. 
While in demand, examples of using modern deep learning techniques for CPD are scarce if the input data are semi-structured like natural language~\cite{iyer2020style}, image sequences or videos~\cite{aminikhanghahi2017survey,truong2020selective}. 

The challenge is that principled criteria related to the change detection delay and the number of false alarms are cumbersome to compute and differentiate (we discuss this point in the section \ref{sec:methods}). Thus, it is impossible to include them in a representation learning framework to train million-parameters neural networks.
Another natural desire from the CPD model is to work in an online mode, as it is the most reasonable and applicable for real-world problems~\cite{xu2019temporal}.

Some natural ad-hoc approaches can deal with simple CPD problems~\cite{aminikhanghahi2017survey}. However, they perform poorly on complex semi-structured data or are restricted while working in online regimes requiring careful hyperparameter selection ~\cite{van2020evaluation}. 
Other advanced methods suitable for processing complex data often do not consider the principled specifics of the CPD problem while training the model and can have significant detection delays and frequent false alarms~\cite{deldari2021tscp2}.

So, we need a framework based on representation learning to fit industrial and scientific needs that comply with the theory related to change point detection~\cite{truong2020selective}.
We expect such an approach to provide a high-quality solution for the CPD problem with semi-structured data, advancing both understanding principles of representation learning for disorder detection and improving deep learning model quality in this area. 

\begin{table}[b]
    \centering
    \caption{Mean performance ranks of considered methods averaged over datasets. See more details below on metrics.}
    \label{tab:ranks}
    \begin{tabular}{lccc}
    \hline
    Metric & AUC & F1 & Cover \\
    \hline
    KL-CPD~\cite{chang2018kernel} & 4.17 & 4.17 & 3.50 \\
    TSCP~\cite{deldari2021tscp2} & 4.67 & 3.83 & 4.66 \\
    BCE & 3 & 2.17 & 2.17 \\
    InDiD (ours) & \textbf{1.5} & \textbf{1.67} & \textbf{1.5} \\
    BCE+InDiD (ours) & \underline{1.67} & \underline{2} & \underline{2}
    \\
    \hline
    \end{tabular}
\end{table}

Our main contributions allow bridging this gap and in more detail are the following:
\begin{itemize}
    \item We present an Instant Disorder Detection (InDiD) framework for the change-point detection for semi-structured data. The method detects changes in the online mode and makes a decision based on the information available at the current moment in case of a single and multiple change points to detect. The scheme of the framework is in Figure~\ref{fig:cpd_scheme}.
    \item To ensure fast and correct change detection, our loss function takes into account both task-specific criteria -- change point detection delay and a number of false alarms -- depicted in Figure~\ref{fig:metrics_figures}. The loss approximates them, making it possible to train large deep models and allow representation learning required for semi-structured data. 
    \item We expand the range of possible applications of CPD algorithms in terms of input data. In particular, we release change-point time labels for datasets that allow the comparison of different approaches in complex scenarios. 
    \item  We conduct a thorough analysis of our approach performance, including investigating an embedding space and quality metrics for different datasets. For considered datasets, our approaches outperform alternatives, as seen from Table~\ref{tab:ranks}.
\end{itemize}

We hope that both our proposed method and labelled datasets would lead to tremendous change in how change points are detected in industrial applications, including various sensors data and video surveillance --- as it allows the processing of such data by deep learning models.

\begin{figure*}[t]
\centering
\begin{subfigure}{.45\textwidth}
 \centering
 \includegraphics[width=0.8\linewidth]{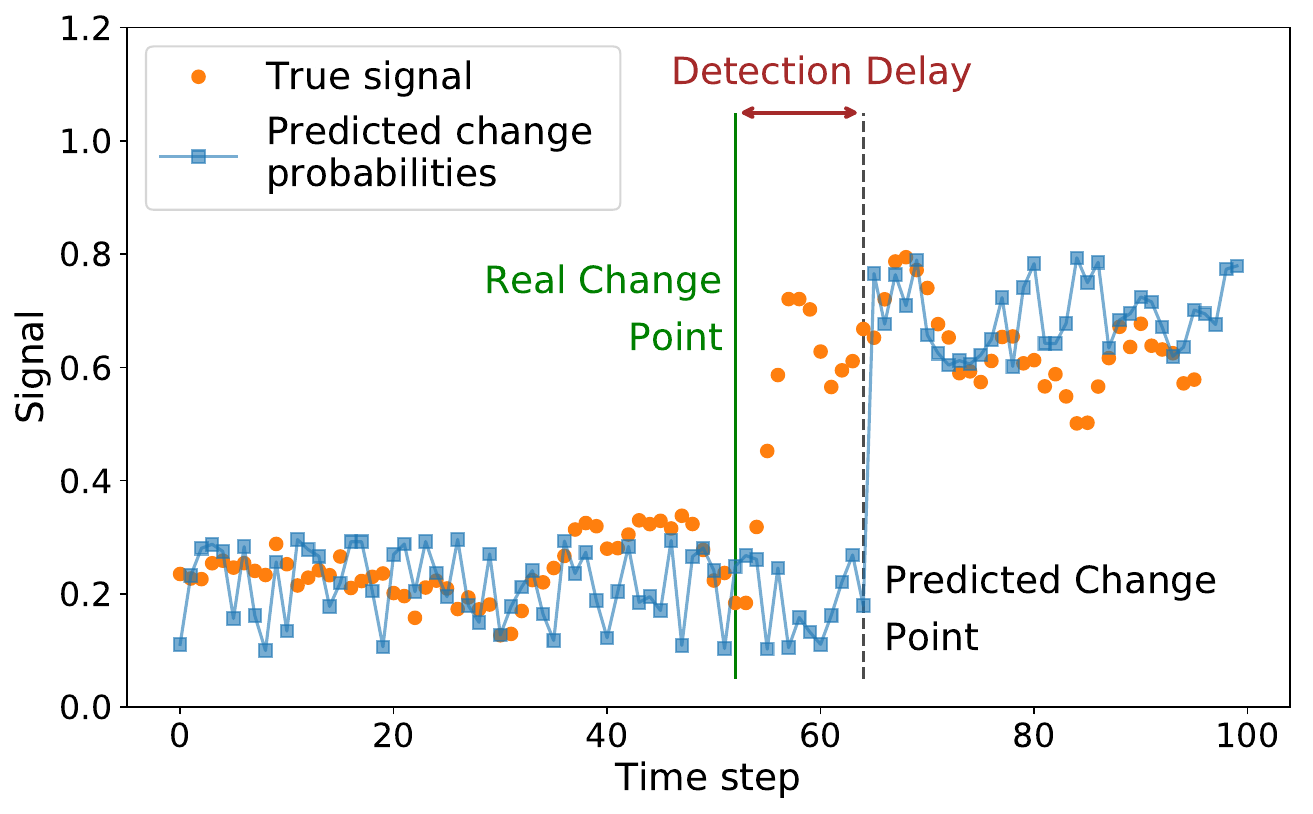}
\end{subfigure}%
\begin{subfigure}{.45\textwidth}
 \centering
 \includegraphics[width=0.8\linewidth]{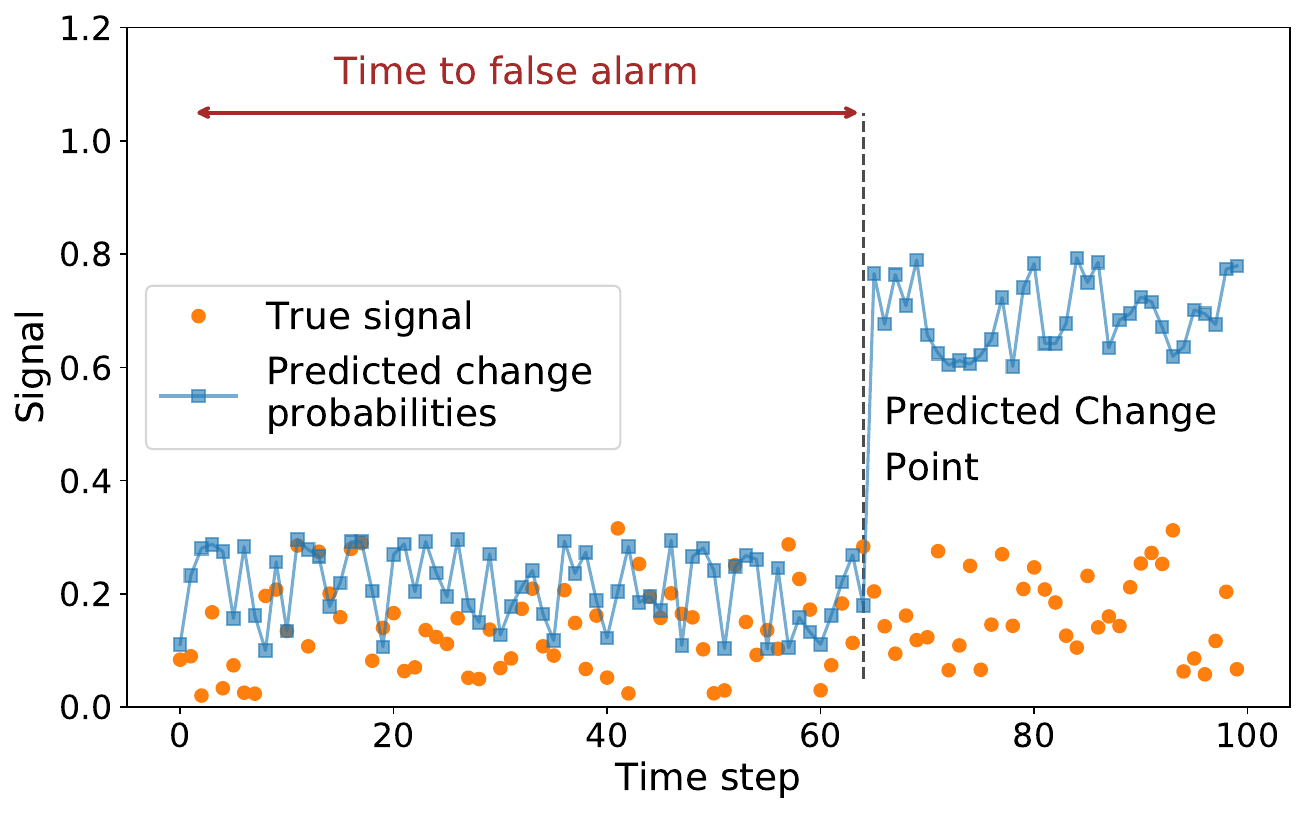}
\end{subfigure}
\caption{Examples of errors during change point detection: delayed change detection (left) and false alarm (right) and corresponding quality metrics Detection Delay and Time to False Alarm. Our approach minimizes a loss function that is a weighted sum of approximations of these errors.}
\label{fig:metrics_figures}
\end{figure*}

\section{Related work}
\label{sec:related_work}

Change Point Detection is a widely studied problem in machine learning and statistics for sequential data. Several approaches to defining the problem exist. 
For example, the classic CPD problem statement considers two natural performance measures: a detection delay and a number of false alarms ~\cite{polunchenko2012state, shiryaev2017stochastic} or their analogues. 
We closely follow this statement while looking at other similar ideas and results.
Below we highlight the main research directions and relevant works.

For strong assumptions, theoretically optimal solutions like CUSUM are available~\cite{tartakovsky2010state}.
However, there are little or no theoretical guarantees in realistic high-dimensional cases.
For example, the authors~\cite{wang2022high} suggest that the detection delay is of order $\sqrt{D}$, where $D$ is a single observation dimension, which is prohibitively high for many real problems.
While additional assumptions like sparsity allow reducing the delay, practitioners go a different road hoping that simpler models can learn to change detection.

More applied offline methods aim to find an optimal partition of a signal~\cite{van2020evaluation, truong2020selective} to a part before and after the change point. 
For this case, one can use approaches based on searching through all possible partitions with or without regularisation terms in cost function~\cite{bai1997estimating, killick2012optimal}. 
Although such methods have been known for a long time, they still are successful for different simple univariate time series~\cite{van2020evaluation} ahead of other more complex approaches like non-trainable Kernel CPD~\cite{harchaoui2009kernel, arlot2019kernel}. 

The comprehensive work~\cite{van2020evaluation} evaluated the most popular CPD methods on several sequences, including a small number of multivariate (with maximum dimension equals 4) and synthetic signals. The authors note that most methods perform poorly for multivariate sequences and require accurate hyperparameter selections. 
For a more in-depth overview with detailed descriptions of CPD methods particularities and performances comparison, interested readers can look into review papers and books~\cite{truong2020selective,aminikhanghahi2017survey, van2020evaluation, xie2021sequential}.

An intuitive method to work with sophisticated structured data streams is to embed them into a low-dimensional space where classic CPD methods are already known to work. In the paper ~\cite{romanenkova2019real}, the authors propose to obtain low-dimensional representations using classic machine learning approaches. On top of it, they apply CPD statistics similar to~\cite{polunchenko2012state}. The authors conclude that we need both a classifier for anomalous/normal data and a change point detector to precisely identify a boundary between classes for an industry-ready solution. A similar idea to use a non-trainable Kernel CPD for improving semi-supervised classification appears in~\cite{ahad2020semi}.

However, the power of classic CPD methods is limited by their inability to \emph{learn} representations of complex data.
The development of deep learning methods allows researchers to consider representation-learning models for CPD tasks via neural networks. A straightforward approach~\cite{hushchyn2020online} uses a similar idea as in CPD based on the probability density ratio but compares representations from neural networks. Another paper adopts deep Q learning to decide on change points~\cite{ma2021deep}.

More formal representation learning aims at obtaining data embeddings by considering the particularities of a task. In relation to sequential data, this task is challenging, as we should take into account the context, but it seems to be the only way to get high-performing models~\cite{ionescu2019object}.
The trainable Kernel CPD (KL-CPD) presented in the paper~\cite{chang2018kernel} uses Maximum Mean Discrepancy (MMD)~\cite{li2015m} for CPD but learns kernels' parameters based on historical data similar to the GAN approach~\cite{goodfellow2014generative}. 
The research ~\cite{deldari2021tscp2} shows that we can apply CPD in low-dimensional space obtained via learning universal representations that close similar objects and divide dissimilar ones. In particular, authors use an analogue of contrastive predictive coding ~\cite{oord2018representation} from a self-supervised area and estimate CP via cosine distance between obtained embeddings. It is important to note that both approaches need to estimate the similarity between two subsequences (so-called history and future), while inference in the online regime, the detection delay depends on the window size.

Unfortunately, the authors validate their approaches only via tabular multivariate time series data,  avoiding video or image sequences datasets. The only exception is ~\cite{hushchyn2020online}, where the authors use sequences based on MNIST with 64 features per time. 

The anomaly detection area provides more research on video processing. The authors of the paper~\cite{sultani2018real} propose a dataset with surveillance video and consider anomaly detection by using multiple instances learning for the dataset with labelling on the whole-video level. The idea is close to the ~\cite{deldari2021tscp2}. They want to bring closer representations of normal frames and move apart the ones of anomaly frames. Another approach from learning video representations area that appeared in the paper \cite{epstein2020oops} is to learn the model to detect the fact of change/anomaly in the small video sub-sequence and then apply it to the full video. While not considering a specific CPD problem statement, the authors of these papers aim at pretty close anomaly detection and use similar datasets. Although these methods are interesting and prospective, they are hard to directly adapt to CPD as they concentrate on anomaly appearances or comparison rather than on the perfect moment identification and do not consider the CPD particularities.

Thus, the existing methods can construct neural-networks-based solutions for anomaly detection and learn corresponding representations. At the same time, they rarely can point at precise change moments online, which is crucial for many applications. Moreover, there are more theoretical-based approaches related to change detection, but they often operate with too restrictive assumptions and have a limited ability to learn complex data representation.

In our work, we unite these two ideas by proposing a principled solution that enables representation learning for deep neural models.
As the amount of available data in such a problem is moderate, we focus on neural networks architectures with a smaller number of parameters like LSTMs~\cite{hochreiter1997long} and GRUs~\cite{cho2014learning} and usage of existing pre-trained models. 

\section{Methods}
\label{sec:methods}

\subsection{Background on the problem statement}

Change point detection aims at the quickest identification of the true change moment $\theta$ in a data stream with a small number of false alarms. The stream of length $T$ is available as sequential observations $X^{1:T} = \{\vecX_i\}_{i=1 }^T$, $\vecX_i \in \mathbb{R}^D$. The desired online CPD procedure should look at the available data up to moment $t \leq T$ and signal about the change point $\tau = t$ or suggest continuing observations. As the described procedure makes a decision using only information available up to the current moment $t$, we work with an online problem statement.

More formally, for $t < \theta$ the data come from a normal-data distribution $p_\infty$, and for $t \geq \theta$ --- from anomalous distribution $p_0$. 
We denote $p_\theta$ the joint data distribution $p(X_{1:T})$ given that the change point is $\theta$. 
Our goal is to find a procedure that produces an estimate $\tau^*$ of the change point. Such estimate should minimize the expected detection delay $\mathbb{E}_\mathbb{\theta} (\tau - \theta)^+$ for the constraint $\mathfrak{M}_a$ such that the Average Time to False alarm is greater than a large value $a$~\cite{polunchenko2012state}:
\begin{gather} 
    \label{eq:ctiteria}
    \tau^* = \underset{\tau \in \mathfrak{M}_a}{ \arg \inf \ }  \mathbb{E}_\mathbb{\theta} (\tau - \theta)^+, 
    \text{where } \mathfrak{M}_a = \{\tau: \mathbb{E}_\theta(\tau |\  \tau<\theta) \geq a \}.  
\end{gather}
Here the expectation $\mathbb{E}_\theta$ with index $\theta$ means that our observations come from the data distribution $p_\theta$ with the true change point at time $\theta$.

Using the method of Lagrange
multipliers with parameter $c \geq 0$, we can rewrite two equations from~\eqref{eq:ctiteria} as an optimization problem with a criterion and solution $\tau^*$:
\begin{gather} 
    \label{eq:criteria_lagrange}
    \mathcal{\tilde{L}}(\tau) \rightarrow \min_\tau \, \mathrm{s.t.} \, \, c(\mathbb{E}_\theta(\tau |\  \tau<\theta) - a) = 0,   \\
    \text{where } \mathcal{\tilde{L}}(\tau) = \mathbb{E}_\mathbb{\theta} (\tau - \theta)^+ - c \mathbb{E}_\theta(\tau |\  \tau<\theta). \nonumber
\end{gather}
Below, we discuss how to formulate a loss function suitable for training neural networks based on the general criterion~\eqref{eq:criteria_lagrange}. 

\subsection{Loss function for InDiD approach}

We consider a set of sequences $D = \{(X_1, \theta_1), ... , (X_N, \theta_N)\}$ with similar length $T$, change points $\theta_i$ are in $\{1, ... , T, \infty\}$. 
Each $X_i$ is a sequence of multivariate observations $\{\vecX_{ij}\}_{j=1}^T$. 

Our goal is to construct a model $f_\vecW$ that minimizes the criterion~\eqref{eq:criteria_lagrange} for the online change point detection procedure. We expect the model $f_\vecW$ to deal with a change point in the following way:
\begin{itemize}
    \item A model $f_\vecW$ outputs a series $\{p^i_t\}^T_{t=1}$, with $p^i_t = f_\vecW(X_i^{1:t})$ based only on the information $X_i^{1:t} = \{\vecX_{ij}\}_{j=1}^t \subset X_i$ available up to the moment $t$. 
    A value $p^i_t$ corresponds to the estimate of the probability of a change point in a sequence $X_i$ at a specified time moment $t$. 
    \item As $p^i_t$ depends only on $X_i^{1:t}$, the model works in an online~mode. 
    \item The obtained distribution $p^i$ provides an estimate for a change moment $\theta$. There are two options for such estimates: (A) select a threshold $p_0$ and signal about the change point if the probability exceeds the threshold $p^i_t > p_0$ or (B) report a change point at $t$ with the probability~$p^i_t$. Scenario (B) considers a probabilistic problem statement, while Scenario (A) uses obtained distribution to make discrete decisions.
\end{itemize}

While theoretically appealing, the loss function $\mathcal{\tilde{L}}(\tau)$ from \eqref{eq:criteria_lagrange} is analytically intractable with complex expectations equivalent to infinite series involved to account for all terms from a moment $t$ up to infinity (see~\ref{sec:proofs} for more details). Besides, as the constraint is often omitted in applied machine learning, we focus only on the minimization  $\mathcal{\tilde{L}}(\tau)$ ignoring conditions. 

We then suggest a new principled loss function for training a model $f_\vecW$ that produces the change probabilities $\{p^i_t\}^T_{t=1}$ for scenario (B):
\begin{equation}
    \label{eq:final_loss}
    \tilde{L}^{\tS}(f_\vecW, D, c) =  \tilde{L}^{\tS}_{delay}(f_\vecW, D) - c \tilde{L}_{FA}(f_\vecW, D).
\end{equation}

It consists of two terms. The first term bounds the detection delay expectations (see ~\ref{sec:proofs} for detailed proofs):
\begin{multline}
    \label{eq:approx_loss_delay}
    \tilde{L}^{\tS}_{delay}(f_\vecW, D) = \frac{1}{N} \sum_{i = 1}^N \tilde{L}^{\tS}_{delay}(f_\vecW, X_i, \theta_i), \\
    \tilde{L}^{\tS}_{delay}(f_\vecW, X_i, \theta_i) = \sum_{t = \theta_i}^{\tS} (t - \theta_i) p^i_{t} \prod_{k = \theta_i}^{t - 1} (1 - p^i_{k}) + (\tS + 1 - \theta_i) \prod_{k = \theta_i}^{\tS} (1 - p_{k}^i),
\end{multline}
where $\tS \leq T$ is a hyperparameter that restricts the size
of a considered segment: for $h \rightarrow \infty$ we get an exact expectation of the detection delay. 

The second term approximates the expected time to false alarm looking at interval with no changes either $[0, \theta_i]$ or $[0, T]$, if there are no change point in $i$-th sequence ($\theta_i = \infty$) (see ~\ref{sec:proofs} for detailed proofs):
\begin{multline}
\tilde{L}_{FA}(f_\vecW, D) = \frac{1}{N}\sum_{i = 1}^N \tilde{L}_{FA}(f_\vecW, X_i, \theta_i), \\
\tilde{L}_{FA}(f_\vecW, X_i, \theta_i) = \sum_{t = 0}^{\tilde{T_i}} t p^i_{t} \prod_{k = 0}^{t - 1} (1 - p^i_{k}) - (\tilde{T_i} + 1) \prod_{k = 0}^{\tilde{T_i}} (1 - p^i_{k}), \\
\text{where } \tilde{T_i} = \min(\theta_i, T). 
\end{multline}

\begin{theorem}\label{main_theorem}
The loss function $\tilde{L}^h(f_\vecW, D, c)$ from~\eqref{eq:final_loss}:
\begin{itemize}
    \item[(1)] is a lower bound for a Lagrangian for $\mathcal{\tilde{L}}(\tau)$ from criteria~\eqref{eq:criteria_lagrange};
    \item[(2)] is differentiable with respect to $p_k^i$ and, thus, $\vecW$.
\end{itemize}
\end{theorem}
The proof for (1) is given in Appendix~\ref{sec:proofs}. The (2) is clear from the equations.

Given the results above, our differentiable loss function is a lower bound for the principled loss function in~\eqref{eq:criteria_lagrange}.
The proposed loss is an accurate approximation to the true one since the high-order terms in the sums have a lower magnitude. Each new term is multiplied by a value of magnitude $p^i_k$ before the change and $(1 - p^i_k)$ -- after it. So, if models are of reasonable quality, we observe an exponential decrease for $k$-th order terms as they are proportional to $(1 - p)^k$ or $p^k$.

\subsection{InDiD approach} 

We update the neural network $f_{\mathbf{W}}$ parameters minimizing the differentiable loss $\tilde{L}^h(f_{\mathbf{W}}, D, c)$ from~\eqref{eq:final_loss}.
The used value $c = \frac{h}{2T}$ that normalizes the number of terms for both parts of the loss performs well empirically for all considered datasets.
As we process inputs sequentially, hidden states for our architecture would represent the current sequence tailored for change point detection.

The overall InDiD method is presented in Algorithm~\ref{alg:indid}.
We either train the model from scratch using the proposed loss to get the pure \emph{InDiD} method or train a classifier with binary cross-entropy loss and then fine-tune it with our loss to get \emph{BCE+InDiD} method. 

\begin{algorithm}[ht]
 \SetKwInOut{Input}{input}\SetKwInOut{Output}{output}
 \Input{$\vecW_0$ --- parameters of a pre-trained model (if available)\\ 
        $k$ --- number of iterations \\
        $c$ --- coefficient for the second term in loss function\\
        $h$ --- number of terms considered in $\tilde{L}^{\tS}_{delay}(f_\vecW, D_{b})$}
 \Output{$\vecW$ - parameters after training}
 $\vecW = \vecW_0$ \\
 \For{$i=0$ to $k$}{
     \For{all batch of objects $D_b$ of size $b$ from dataset $D$}{
  1. Calculate the loss function $\tilde{L}^{\tS}(f_\vecW, D_b, c)$ \\
  2. Get the derivatives of the loss $\frac{\partial \tilde{L}^{\tS}(f_\vecW, D_b, c)}{\partial \vecW}$ \\
  3. Update parameters $\vecW$ via Adam using $\frac{\partial \tilde{L}^{\tS}(f_\vecW, D_b, c)}{\partial \vecW}$
  }
 }
 \caption{Our InDiD algorithm}
 \label{alg:indid}
\end{algorithm}

\subsection{Baselines}

We compare the proposed approach with several classical baselines and the methods based on representation learning techniques. 

As classical methods, we use \emph{BinSeg, PELT} and \emph{KernelCPD} from the ruptures package~\cite{truong2020selective}. We apply these methods for simpler CPD problems without the need to learn complex representations. 

More advanced strategies are connected with representation learning via neural networks. The first natural baseline for competing with the \emph{InDiD} approach is to use a seq2seq binary classification:
we want to classify each time step $t$ as a moment before or after a change to get the change point probabilities $p^t_i$. To train a model that predicts these probabilities, we use the binary cross-entropy loss \emph{BCE}. A model \emph{BCE} looks at all data available at the current moment with $p^i_t = f_\vecW(X_i^{1:t})$. When we combine the BCE approach with InDiD, we get \emph{BCE+InDiD}.

We also consider KL-CPD~\cite{chang2018kernel} and TSCP~\cite{deldari2021tscp2} that learn representations from scratch or use the pre-trained ones.
These methods outperform their deep learning
and non-deep learning CPD competitors~\cite{deldari2021tscp2}.

The theoretical complexity of the considered representation-based method is proportional to the sequence length. For the interested reader, we also present running time in Appendix~\ref{sec:implementation}.

\section{Experiments}
\label{sec:experiments}

In this section, we demonstrate how our change-point detection approach works in real-data scenarios and compare it with the others. 
The main results are given for all datasets. 
Due to space limitations, we consider more specific findings only via one dataset; as for the other datasets, the results were almost similar in our experiments.
The code and data are published online\footnote{The code and data labelling are available at \url{https://github.com/romanenkova95/InDiD}}

\subsection{Comparison details} 

Due to the space limitations, we give the majority of the specific implementation details, including particular architectures of considered methods in the \ref{sec:implementation}. Generally, we empirically evaluated several architectures and training hyperparameters to choose the compromise between such significant characteristics for real-world application as the number of network parameters, training stability and the overall performance. At the same time, for InDiD, BCE, and BCE+InDiD, we train a neural network with the same selected architecture and stopping criteria with no changes except loss type. 

We mostly follow papers ~\cite{chang2018kernel, deldari2021tscp2} for baselines KL-CPD and TSCP and do not change the original architectures, only varying their parameters for maximizing the F1-score. The critical point for the performance of KL-CPD is its GAN nature. We have faced huge network computational complexity because of the necessity of coding and decoding multi-dimensional sequences in this approach. To properly compare KL-CPD and TSCP with other considered methods, we transform the models' scores (MMD score and cosine similarity, respectively) to [0, 1] interval via hyperbolic tangent and sigmoid functions.  

Another important note for all models concerns video datasets, where we input not original videos, but their representations obtained via a pre-trained 3D Convolutional network~\cite{feichtenhofer2019slowfast}. We investigated different ways of training models on video --- and this approach leads to better models. 

For classic CPD approaches, hyperparameters come from a grid-search procedure that maximizes the F1 metric. We varied the method's central models (or kernel for Kernel CPD), the penalty parameter and the number of change points to detect in every sequence. We provide results for the best classic model, while the detailed comparison of classic CPD approaches is in Appendix~\ref{sec:classic}.

\begin{table*}[t!]
\caption{Main quality metrics for considered CPD approaches. $\uparrow$ marks metrics we want to maximize, $\downarrow$ marks metrics we want to minimize. The results are averaged by 5 runs and are given in the format $mean\pm std$. Best values are highlighted with \textbf{bold} font, second best values are \underline{underlined}.} \label{tab:results_metrics}
\centering
\begin{tabular}{lccccc}
\hline
Method & Mean Time & Mean DD $\downarrow$ & AUC $\downarrow$ & F1 $\uparrow$ & Covering $\uparrow$ \\
      & to FA $\uparrow$ &  &  &  &  \\
\hline
\multicolumn{6}{c}{1D Synthetic data} \\
\hline
Best classic method & 94.81 & 0.64 & na & 0.9872 & \textbf{0.9954} \\
KL-CPD & \underline{95.21 $\pm$ 0.20} & 1.68 $\pm$ 0.10 & 646.44 $\pm$ 1.01 & 0.9806 $\pm$ 0.0000 & 0.9857 $\pm$ 0.0000 \\
TSCP & \textbf{104.20 $\pm$ 1.80} & 13.22 $\pm$ 2.36 & 1333.92 $\pm$ 72.96 & 0.8430 $\pm$ 0.0551 & 0.9320 $\pm$ 0.0223 \\ 
BCE & 94.49 $\pm$ 0.00 & \textbf{0.50 $\pm$ 0.02} & 606.00 $\pm$ 4.91 & \textbf{0.9904 $\pm$ 0.0000} & \underline{0.9941 $\pm$ 0.0002}
\\
InDiD (ours) & 94.41 $\pm$ 0.13 & \underline{0.54 $\pm$ 0.04} & \textbf{598.70 $\pm$ 5.00} & \underline{0.9898 $\pm$ 0.0010} & 0.9938 $\pm$ 0.0004 \\
BCE+InDiD (ours) & 94.59 $\pm$ 0.17 & 0.61 $\pm$ 0.13 & \underline{605.30 $\pm$ 5.12} & 0.9891 $\pm$ 0.0021 & 0.9936 $\pm$ 0.0006 \\
\hline
\multicolumn{6}{c}{100D Synthetic data} \\
\hline
Best classic method & 94.20 & \textbf{0.03} & na & \textbf{0.9968} & \textbf{0.9996} \\
KL-CPD & \underline{98.13 $\pm$ 1.79} & 5.85 $\pm$ 1.85 & 782.84 $\pm$ 46.23 & 0.9201 $\pm$ 0.0308 & 0.9557 $\pm$ 0.0009 \\
TSCP & \textbf{104.68 $\pm$ 2.26} & 15.39 $\pm$ 1.36 & 1390.71 $\pm$ 45.22 & 0.7530 $\pm$ 0.0485 & 0.9111 $\pm$ 0.0157 \\ 
BCE & 94.20 $\pm$ 0.00 & \textbf{0.03 $\pm$ 0.00} & \underline{573.39 $\pm$ 0.22} & \textbf{0.9968 $\pm$ 0.0000} & \textbf{0.9996 $\pm$ 0.0001} \\
InDiD (ours) & 94.20 $\pm$ 0.00 & \textbf{0.03 $\pm$ 0.01} & 573.47 $\pm$ 0.31 & \textbf{0.9968 $\pm$ 0.0000} & \textbf{0.9996 $\pm$ 0.0001} \\
BCE+InDiD (ours) & 94.20 $\pm$ 0.00 & \textbf{0.03 $\pm$ 0.00} & \textbf{573.29 $\pm$ 0.16} & \textbf{0.9968 $\pm$ 0.0000} & \textbf{0.9996 $\pm$ 0.0000} \\
\hline
\multicolumn{6}{c}{Human Activity Recognition} \\
\hline
Best classic method & 5.90 & \textbf{0.07} & na & 0.6703 & 0.8538 \\
KL-CPD & 10.29 $\pm$ 0.16 & 1.00 $\pm$ 0.00 & 52.12 $\pm$ 0.79 & 0.9413 $\pm$ 0.0065 & 0.8856 $\pm$ 0.0065 \\
TSCP & 10.54 $\pm$ 0.09 & 2.71 $\pm$ 0.09 & 69.53 $\pm$ 0.53 & 0.9168 $\pm$ 0.0032 & 0.7814 $\pm$ 0.0677 \\ 
BCE & 11.00 $\pm$ 0.16 & \underline{0.20 $\pm$ 0.08} & 43.22 $\pm$ 2.03 & \underline{0.9886 $\pm$ 0.0068} & 0.9851 $\pm$ 0.0125 \\
InDiD (ours) & \underline{11.36 $\pm$ 0.13} & 0.32 $\pm$ 0.08 & \textbf{41.54 $\pm$ 0.23} & 0.9851 $\pm$ 0.0050 & \underline{0.9975 $\pm$ 0.0013} \\
BCE+InDiD (ours) & \textbf{11.42 $\pm$ 0.02} & 0.36 $\pm$ 0.00 & \underline{41.57 $\pm$ 0.11} & \textbf{0.9893 $\pm$ 0.0006} & \textbf{0.9992 $\pm$ 0.0006} \\
\hline
\multicolumn{6}{c}{Sequences of MNIST images} \\
\hline
Best classic method & 43.02 & 4.08 & na & 0.5500 & 0.8364 \\
KL-CPD & \textbf{49.03 $\pm$ 1.60} & 6.18 $\pm$ 1.21 & 219.65 $\pm$ 2.64 & 0.6522 $\pm$ 0.0306 & 0.8543 $\pm$ 0.0060 \\
TSCP & \underline{46.87 $\pm$ 1.09} & 5.75 $\pm$ 0.67 & 295.58 $\pm$ 1.85 & 0.6923 $\pm$ 0.0205 & 0.8494 $\pm$ 0.0025 \\ 
BCE & 44.94 $\pm$ 0.04 & 3.37 $\pm$ 0.23 & 237.94 $\pm$ 8.54 & \textbf{0.9862 $\pm$ 0.0034} & 0.9120 $\pm$ 0.0047 \\
InDiD (ours) & 44.86 $\pm$ 0.32 & \underline{2.15 $\pm$ 0.68} & \underline{213.76 $\pm$ 10.77} & \underline{0.9026 $\pm$ 0.0285} & \underline{0.9392 $\pm$ 0.0012} \\
BCE+InDiD (ours) & 44.66 $\pm$ 0.06 & \textbf{1.47 $\pm$ 0.24} & \textbf{202.19 $\pm$ 2.07} & 0.8866 $\pm$ 0.0217 & \textbf{0.9527 $\pm$ 0.0047} \\
\hline
\multicolumn{6}{c}{Explosions} \\
\hline
KL-CPD & \underline{15.64 $\pm$ 0.18} & 0.25 $\pm$ 0.04 & 1.51 $\pm$ 0.30 & 0.2785 $\pm$ 0.0381 & 0.97045 $\pm$ 0.0069 \\
TSCP & 15.07 $\pm$ 0.13 & \underline{0.15 $\pm$ 0.01} & 1.35 $\pm$ 0.09 & 0.3481 $\pm$ 0.0252 & 0.9515 $\pm$ 0.0062 \\ 
BCE & 14.74 $\pm$ 0.47 & \textbf{0.13 $\pm$ 0.03} & 0.82 $\pm$ 0.14 & 0.3094 $\pm$ 0.0881 & 0.9728 $\pm$ 0.0084 \\
InDiD (ours) & \textbf{15.73 $\pm$ 0.11} & 0.20 $\pm$ 0.04 & \textbf{0.52 $\pm$ 0.13} & \textbf{0.5298 $\pm$ 0.0482} & \textbf{0.9858 $\pm$ 0.0013} \\
BCE+InDiD (ours) & 15.35 $\pm$ 0.24 & 0.18 $\pm$ 0.03 & \underline{0.80 $\pm$ 0.12} & \underline{0.3839 $\pm$ 0.0353} & \underline{0.9802 $\pm$ 0.037} \\
\hline
\multicolumn{6}{c}{Road Accidents } \\
\hline
KL-CPD & \textbf{15.26 $\pm$ 0.53} & 0.91 $\pm$ 0.08 & 7.29 $\pm$ 0.90 & 0.1060 $\pm$ 0.0462 & 0.9220 $\pm$ 0.0020 \\
TSCP & 14.43 $\pm$ 0.17 & 0.67 $\pm$ 0.05 & 6.15 $\pm$ 0.31 & 0.2348 $\pm$ 0.0290 & 0.8999 $\pm$ 0.0044 \\ 
BCE & 13.83 $\pm$ 0.54 & \underline{0.53 $\pm$ 0.10} & 4.32 $\pm$ 0.74 & 0.1760 $\pm$ 0.0650 & \underline{0.9264 $\pm$ 0.0143} \\
InDiD (ours) & \underline{14.96 $\pm$ 0.31} & 0.66 $\pm$ 0.08 & \textbf{3.11 $\pm$ 0.09} & \textbf{0.2865 $\pm$ 0.0279} & \textbf{0.9349 $\pm$ 0.0064} \\
BCE+InDiD (ours) & 13.87 $\pm$ 0.52 & \textbf{0.51 $\pm$ 0.09} & \underline{4.14 $\pm$ 0.88} & \underline{0.2476 $\pm$ 0.0166} & 0.9173 $\pm$ 0.0063 \\
\hline
\end{tabular}
\end{table*}

\begin{figure*}[!ht] 
\centering
\begin{subfigure}{.28\textwidth}
 \centering
 \includegraphics[width=\linewidth]{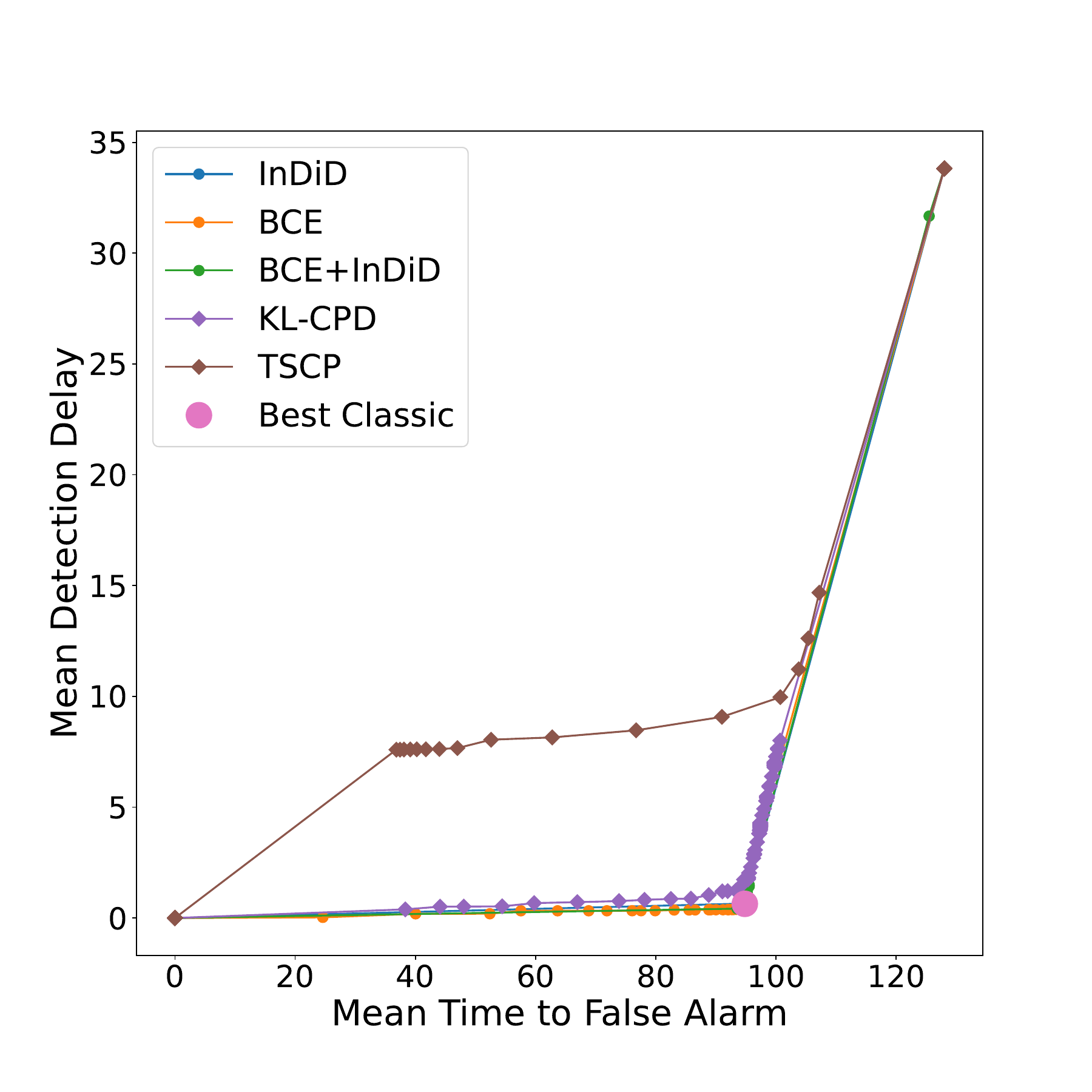}
 \subcaption{Normal 1D}
\end{subfigure}%
\begin{subfigure}{.28\textwidth}
 \centering
 \includegraphics[width=\linewidth]{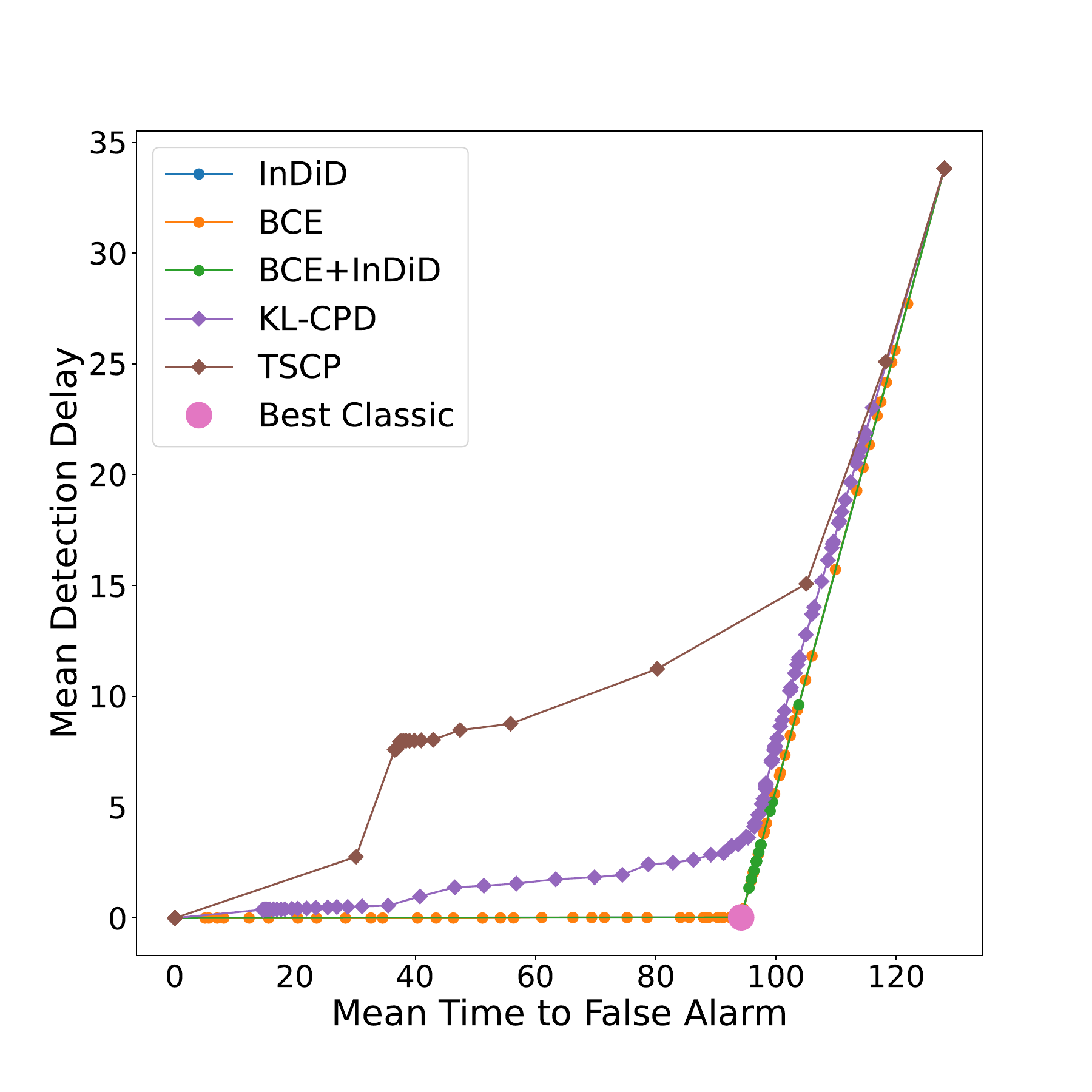}
 \subcaption{Normal 100D}
\end{subfigure}
\begin{subfigure}{.28\textwidth}
 \centering
 \includegraphics[width=\linewidth]{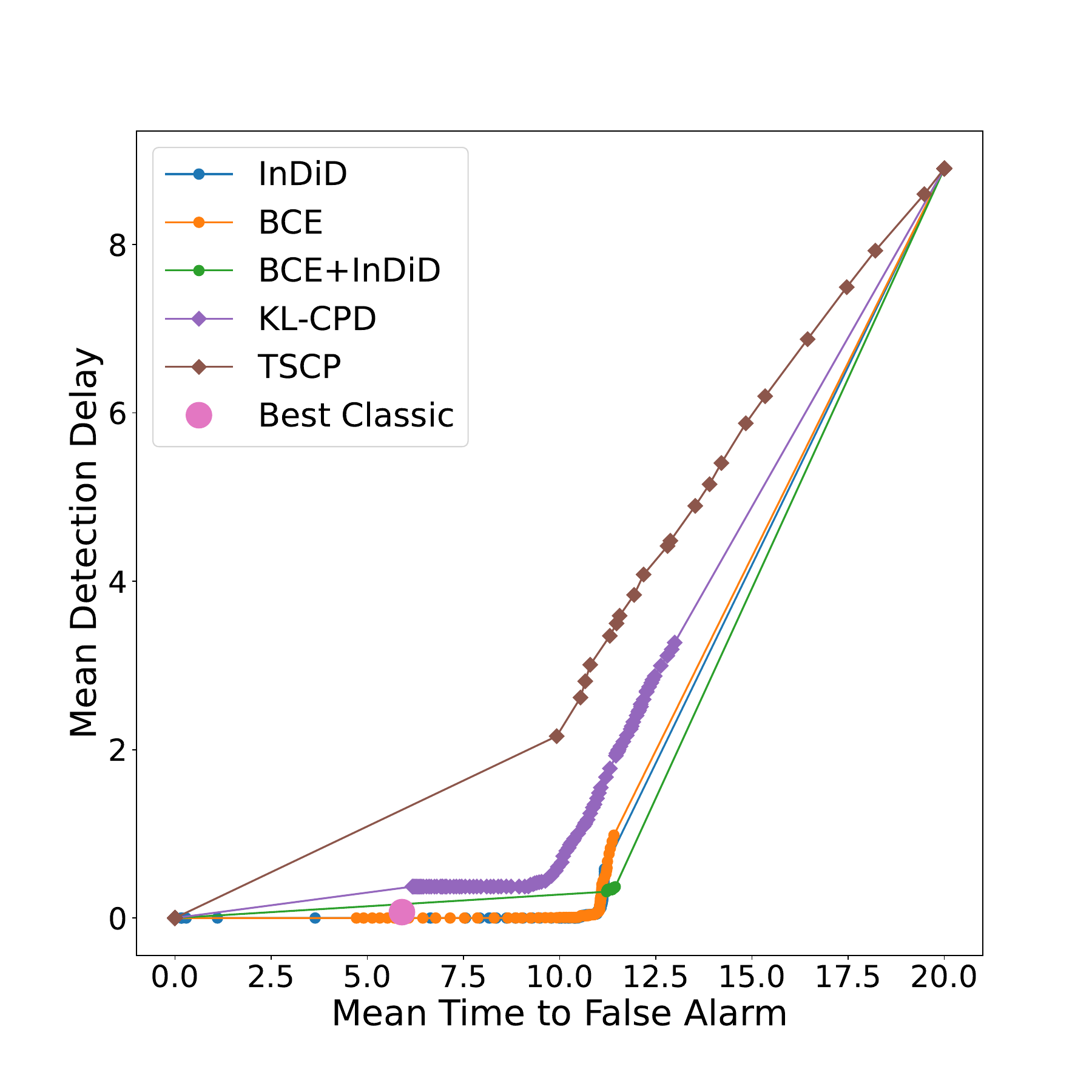}
 \subcaption{Human Activity}
\end{subfigure}
\\
\begin{subfigure}{.28\textwidth}
 \centering
 \includegraphics[width=\linewidth]{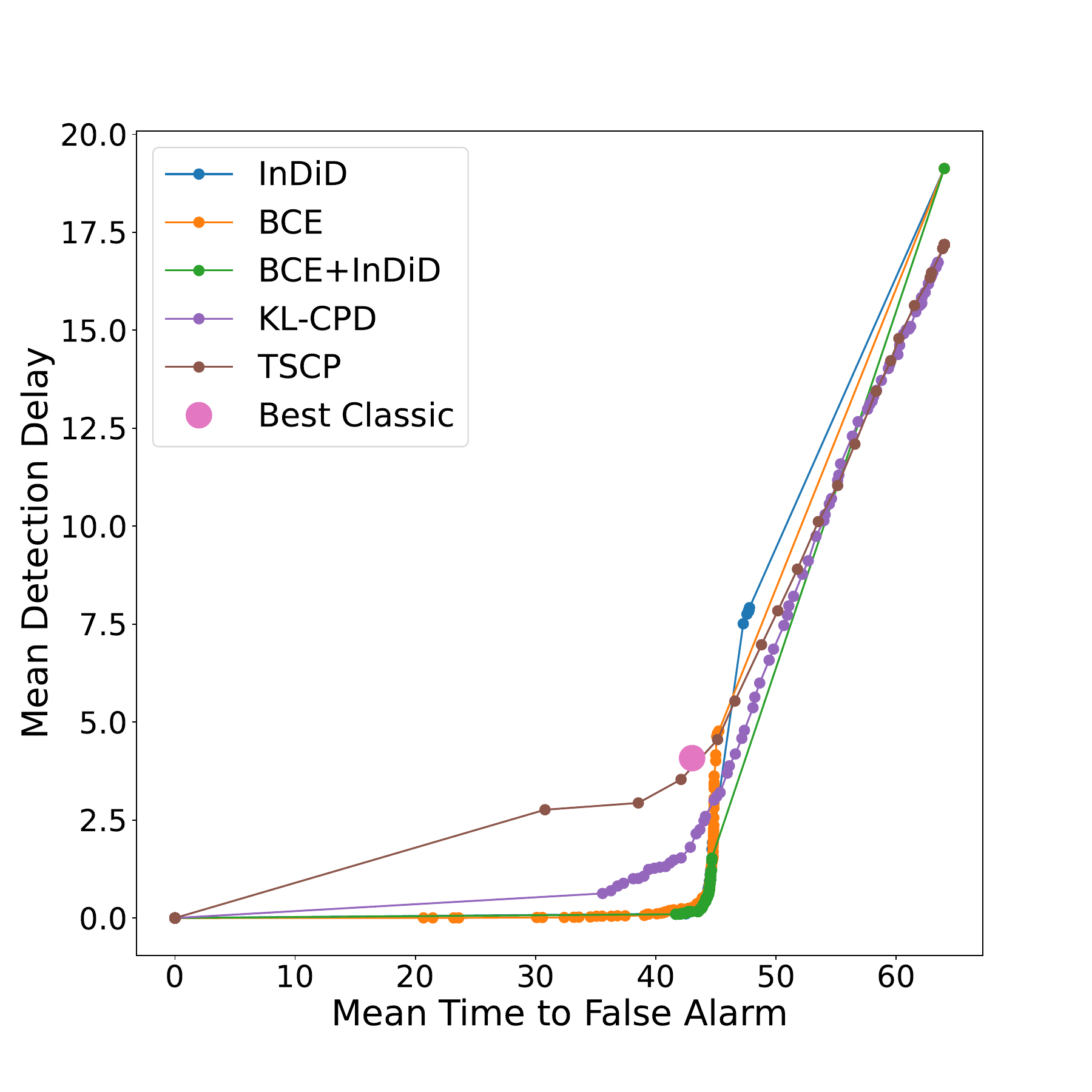}
 \subcaption{MNIST}
\end{subfigure}%
\begin{subfigure}{.28\textwidth}
 \centering
 \includegraphics[width=\linewidth]{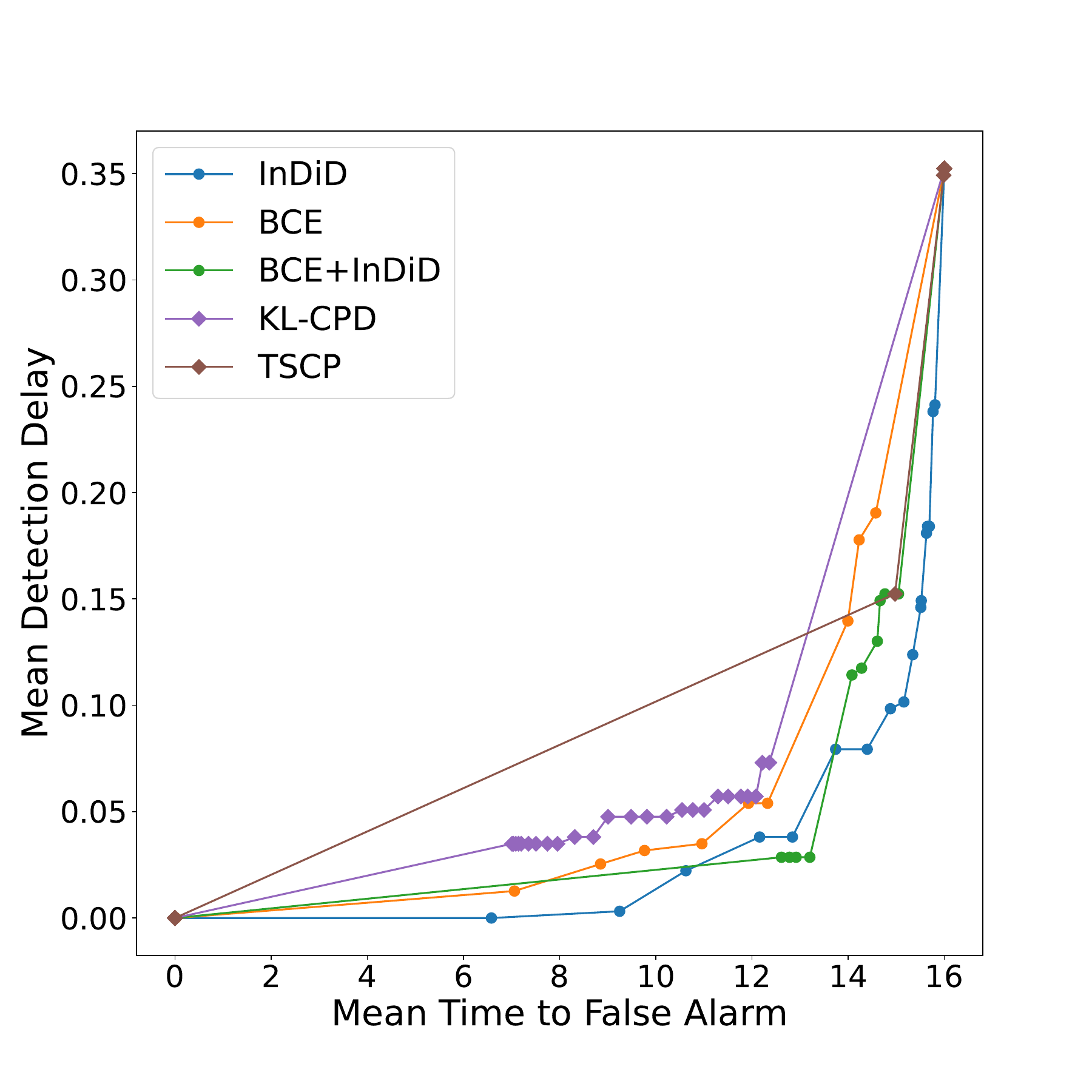}
 \subcaption{Explosion}
\end{subfigure}
\begin{subfigure}{.28\textwidth}
 \centering
 \includegraphics[width=\linewidth]{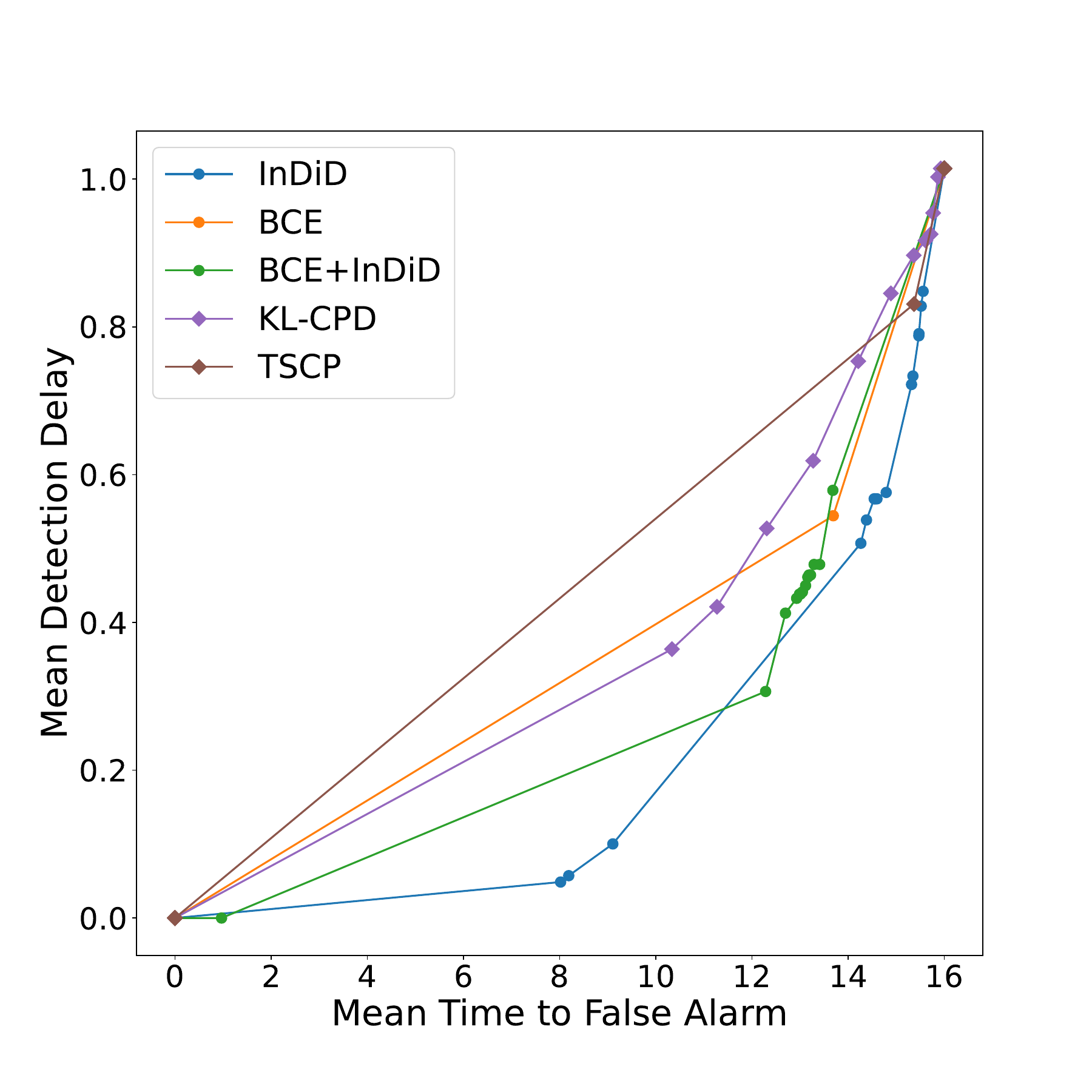}
 \subcaption{Road accidents}
\end{subfigure}
\caption{The performance detection curves for various datasets. If a method does not have a threshold to select, we plot a single point corresponding to its performance. The lower and righter the curve lies, the better. Thus, the models trained with our losses have better results on all types of data except very simple cases where the performance is similar.}
\label{fig:curves}
\end{figure*}

\subsection{Datasets}

Our datasets have various complexity: from multivariate data from sensors to video datasets. Detailed information on data and their preprocessing are given in Appendix~\ref{sec:data_details}.

\textbf{Synthetic sequences.} We start with two toy examples: 1D and 100D Gaussians with a change in mean and without it. 
We also generate a Synthetic 1D dataset with the number of change points ranging from $0$ to $9$ to check how different approaches perform in multiple change-point detection scenarios.

\textbf{Human Activity Dataset.} As a dataset with numerical sequences, we use the USC-HAD dataset \cite{zhang2012usc} with $12$ types of human activities.  We sample subsequences with changes in the type of human activity and without them. 
Each subsequence consists of $28$ selected measurements during $20$ time ticks. 

\textbf{Sequences of MNIST images.} We generate another dataset based on MNIST~\cite{lecun-mnisthandwrittendigit-2010} images. With a Conditional Variational Autoencoder (CVAE), ~\cite{sohn2015learning} we construct sequences with and without a change. Normal sequences start and end with two images from similar classes that smoothly transform one into another. Sequences with a change start and end with ones from different classes.
We generated a balanced dataset of 1000 sequences with a length 64.

\textbf{Explosions and Car Accidents.} UCF-Crime is a video dataset for anomaly labelling of  sequences~\cite{sultani2018real}. It consists of real-world $240\times 320$ videos, with 13 realistic anomaly types such as explosions, road accidents, burglary, etc., and normal examples.
We consider two types of anomalies that correspond to disorders: explosions and car accidents.
\emph{Explosions} lead not to a point anomaly but to long-lasting consequences with a distribution change. 
\emph{Road Accidents} follow a similar pattern but are more challenging. We use it to test the limit of applicability of our approaches. The number of frames in considered sequences is $16$.
We carefully labelled frames related to the change point and provided this labelling along with our code. Thus, these data are ready to use by a community.

\begin{figure}[!ht]
\centering
 \includegraphics[width=0.6\linewidth]{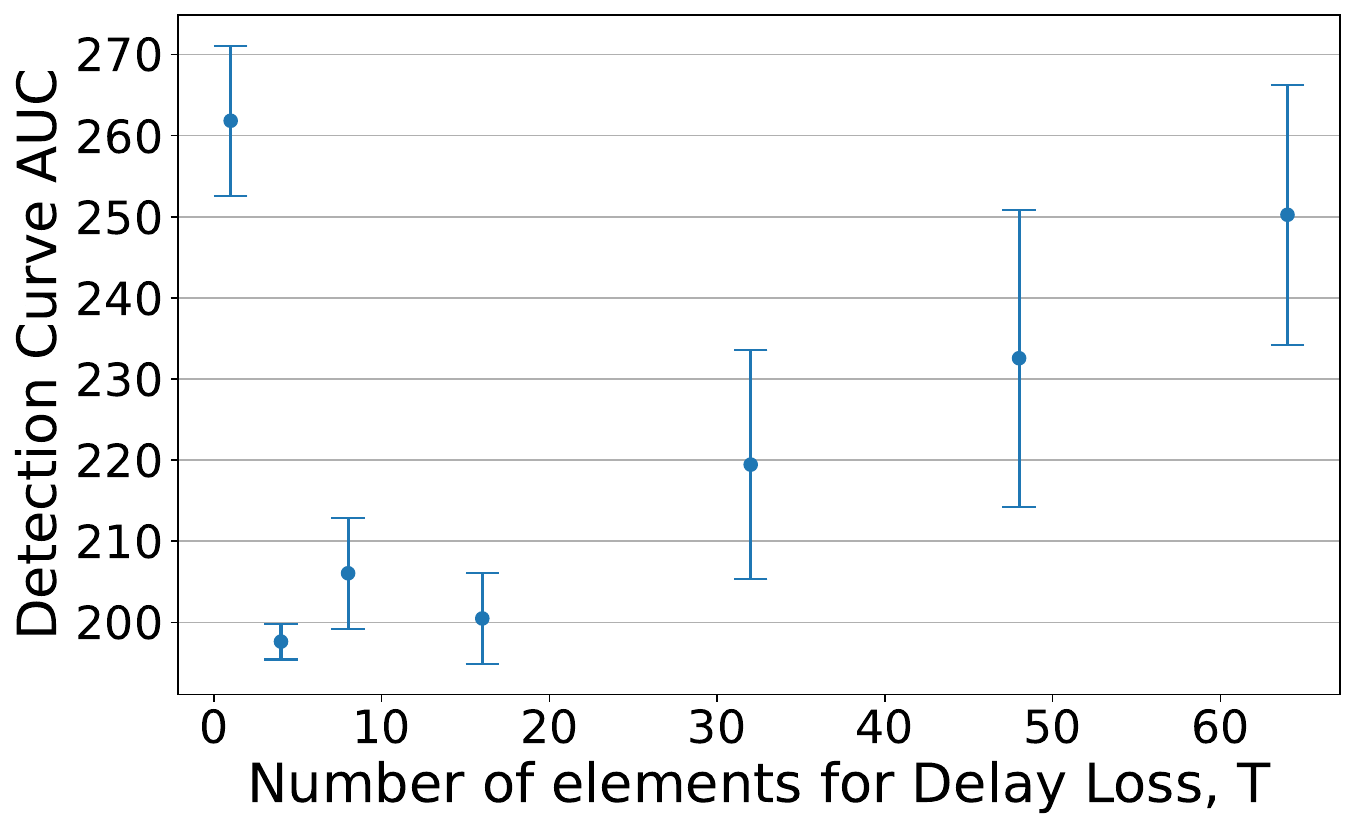}
\caption{The model's performance dependence on the number of elements for Delay Loss for MNIST data for InDiD averaged over $3$ runs.}
\label{fig:alphas_and_ws}
\end{figure}

\begin{figure*}[h!]
\centering
\begin{subfigure}{.3\textwidth}
 \centering
 \includegraphics[width=0.9\linewidth]{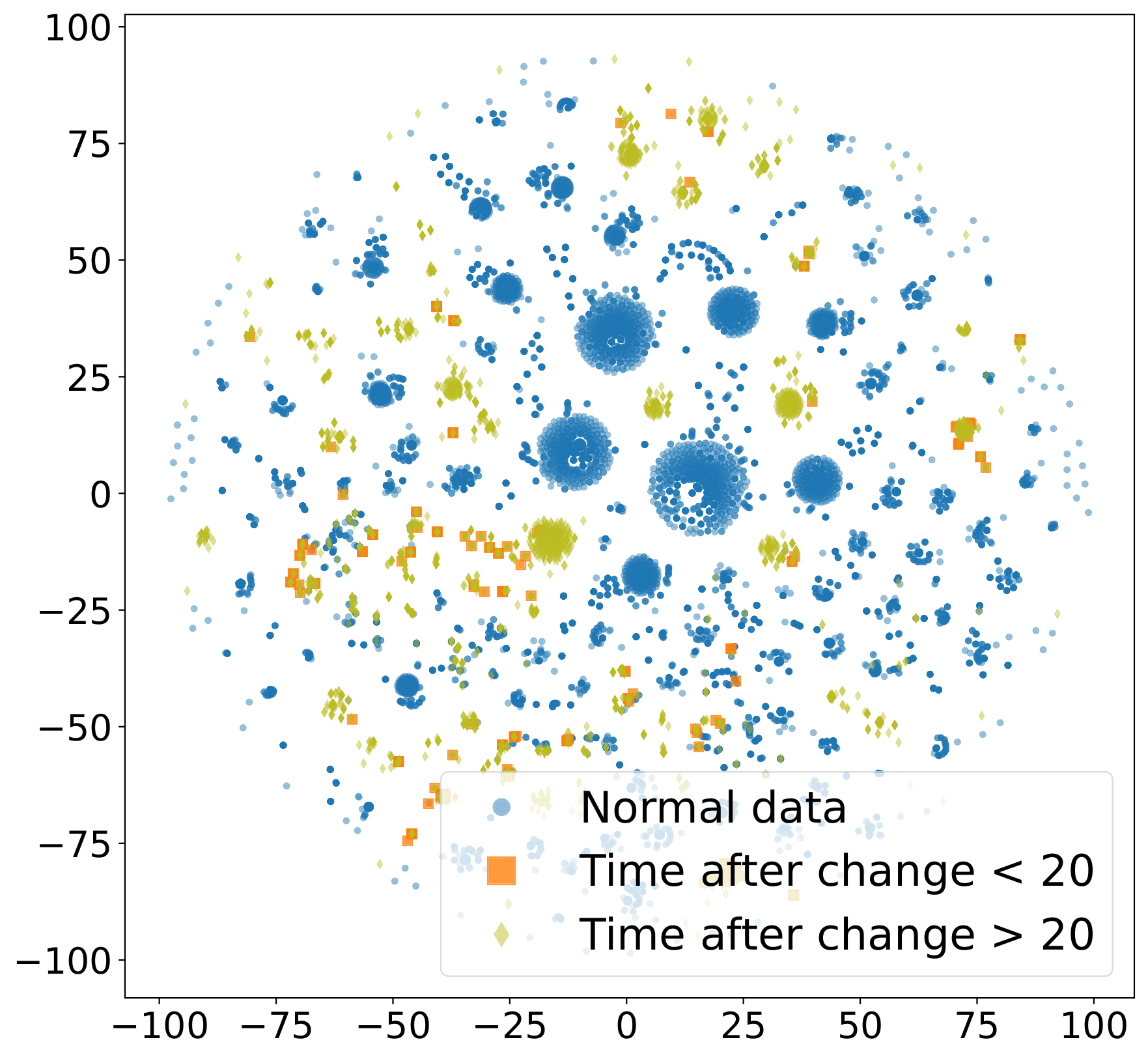}
\end{subfigure}%
\begin{subfigure}{.3\textwidth}
 \centering
 \includegraphics[width=0.9\linewidth]{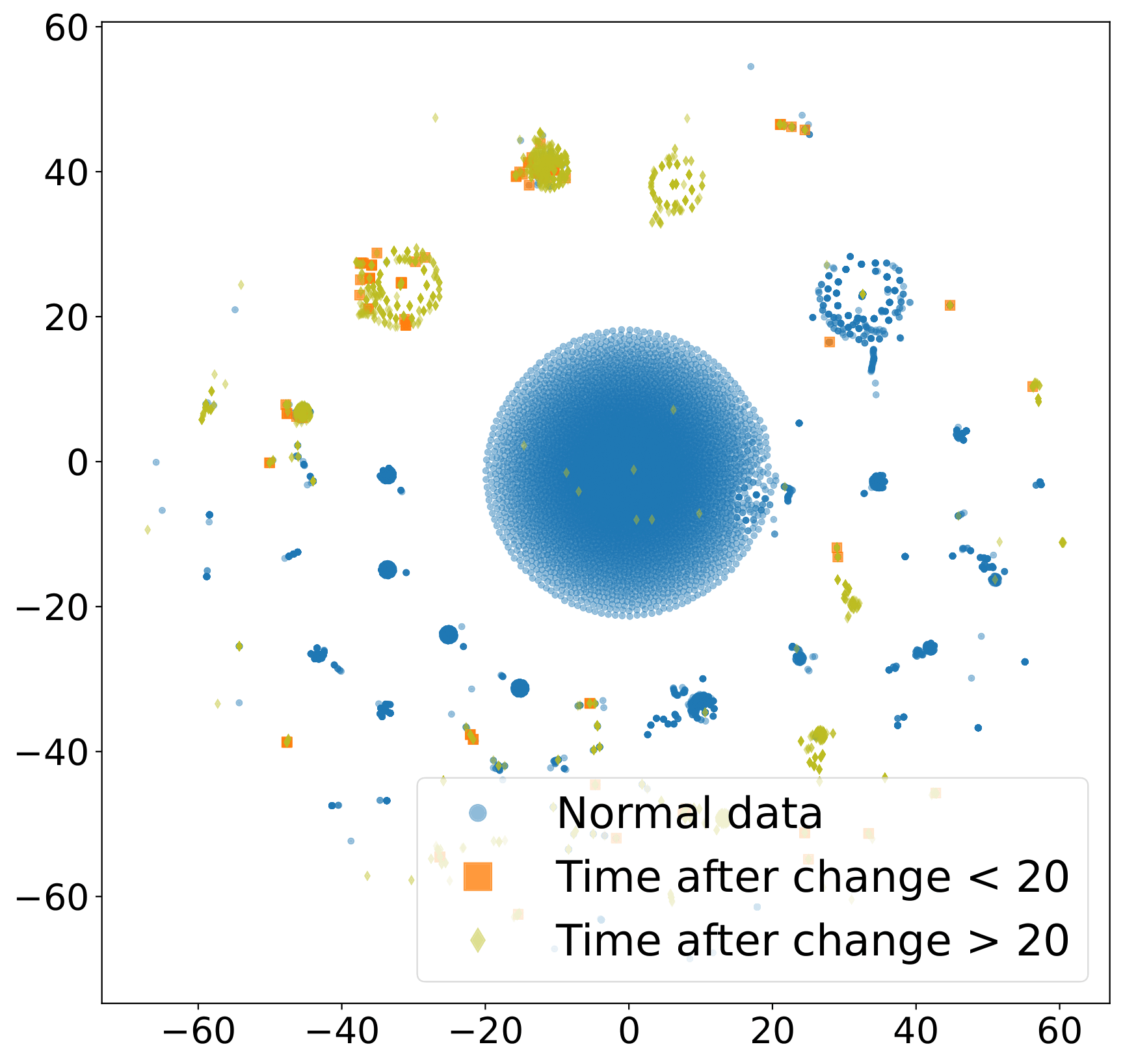}
\end{subfigure}
\begin{subfigure}{.3\textwidth}
 \centering
 \includegraphics[width=0.9\linewidth]{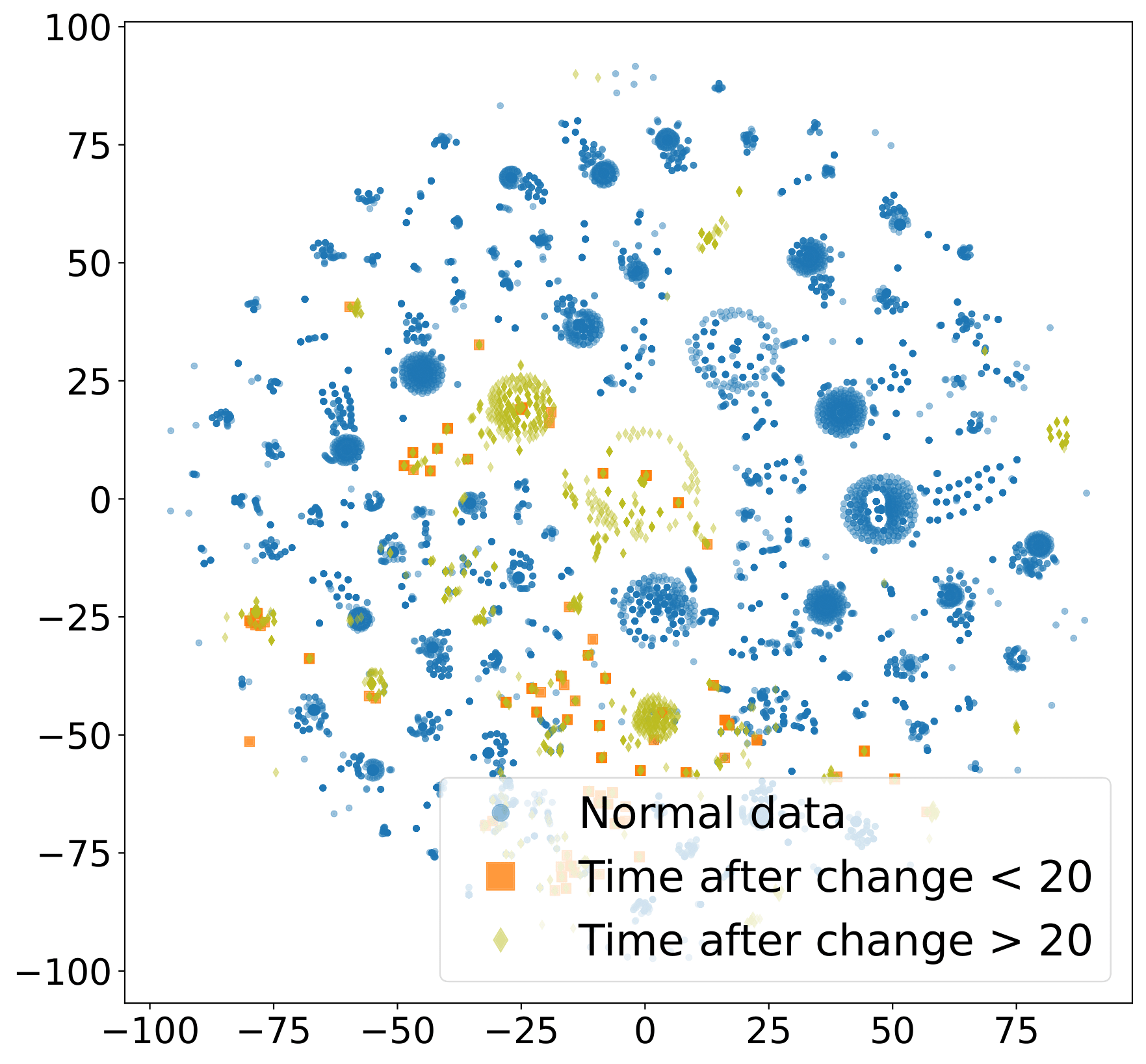}
\end{subfigure}
\caption{tSNE for embeddings obtained via approaches with the binary cross-entropy loss (left), InDiD with our loss (center), and combined InDiD+BCE approach (right).}
\label{fig:tsne_embeddings}
\end{figure*}

\subsection{Benchmarks and Metrics}
\label{sec:metrics}

\paragraph{Classification quality metrics.} 
For our problem statement, elements of the confusion matrix have the following meaning.
For True Positive (TP), we correctly detect a change no earlier than it appears. 
True Negative (TN) indicates that we rightly do not detect anything. 
False Negative (FN) defines a case when we miss the disorder while it appears in the real data. The most significant is the False Positive (FP). The obvious case of FP is predicting changes for normal signals. In addition, we detect the FP when the model predicts change before it appears in a sequence with the real disorder.
Based on these metrics, we calculate $F1 = \frac{TP}{TP + 0.5(FP + FN)}$. Our estimation is slightly different from the common F1 from~\cite{van2020evaluation}, as we penalize premature detections. We expect that such evaluation leads to better real-life CPD application estimation, as there is no reason for the model to worry about a signal that comes from a normal regime. 

\paragraph{Detection Delay and Time to False Alarm.} For a quantitative comparison of CPD approaches, we concentrate on typical metrics for change point 
detection: Detection Delay and Time to False Alarm depicted in  Figure~\ref{fig:metrics_figures}. We want to meet two competing goals: minimize the Detection Delay and maximize the Time to False Alarm. 
The Detection Delay $(\tau - \theta)^+$ is the difference between the true change point $\theta$ and the model’s prediction of change point $\tau$. The Time to False Alarm $\tau$ is the difference in time between a first false positive prediction if it exists and the sequence start. 

\paragraph{Area under the detection curve (AUC).} We warn about the change the first time when the change probability prediction by model $p_t$ exceeds a threshold $s$. So, varying $s$, we obtain different trade-offs between mean detection delay ($x$-axis) and mean time to a false alarm ($y$-axis) and get the detection curve. We want to find an approach that minimizes the area under this curve.
As this metric provides an overall performance evaluation, we consider it as the main quality metric for comparison.

\paragraph{Covering metric}

As a field-specific metric for CPD, we also examine Covering metric presented in~\cite{van2020evaluation}. It follows evaluation for the image segmentation problem, as we look at similarity of sequence partitions based on a predicted partition $G'$ and a true one $G$, i.e. $ \mathrm{Covering}(G, G') = \dfrac{1}{T} \sum_{A \in G} |A| \max_{A' \in G'} \frac{|A \cap A'|}{|A \cup A'|}$. In our case, we consider partition to parts before and after a change point. The higher the Covering is, the better the performance is.

\subsection{Main results}

Table~\ref{tab:results_metrics} presents metrics for various datasets. We consider the most important metrics F1 and AUC, as well as the other metrics. 

Our proposed methods outperform the baselines in terms of these metrics while having a pretty low and even the lowest detection delay. It is important to remember that KL-CPD and TSCP have predetermined detection delays regarding window size due to the methods' nature. So, these methods often have a greater Mean Detection Delay. In addition, these methods seem to produce many false positives that can be identified from a huge Mean Time to False Alarm but low F1 scores. For simple data like Synthetic 1D and Synthetic 100D, the results of all the methods, even classic, are similar to each other. With increasing data complexity, the baselines start to degrade. 

A similar conclusion on how threshold selection affects quality metrics for considered change detection approaches is provided by Figure~\ref{fig:curves}. 

\subsection{Model understanding}
\paragraph{Representations analysis}

To check the quality of obtained representations, we get tSNE embeddings for different methods and draw them in Figure~\ref{fig:tsne_embeddings}. We show embeddings for MNIST data before, immediately after a change point and later on. The model trained with BCE loss does not bother about fast change detection. Embeddings for points immediately after the change are close to those before the change. Thus, we have a significant detection delay.
For our InDiD approach, embeddings for figures after a change rapidly move to a separate region, making change point detection faster and more accurate. Using a combined BCE+InDiD approach allows us to leverage both method benefits and obtain more diverse representations via BCE loss but separated for faster change detection via InDiD. 

\paragraph{Number of terms in loss}
We consider how many terms we need to obtain a reasonable approximation of the true Lagrangian that leads to a good model.
Our hypothesis (see ~\ref{sec:methods}) confirms. As we see in Figure~\ref{fig:alphas_and_ws}, the proposed approximation is pretty accurate, and we do not need a large number of components in it. 
On the other hand, the inclusion of a larger number of components does not harm the numerical stability of our approach.

\paragraph{Multiple CPD}
Additionally, we investigated the performance of the considered models on synthetic 1D datasets with multiple change points (see Table~\ref{table:multiple_metrics}). We train various models on the original Synthetic 1D data with one change in the distribution and then apply the models to the sequences with multiple disorders. 
InDiD still outperforms the same model with BCE loss and the other baselines.

\begin{table}[!ht]
    \centering
    \caption{Maximum value of Covering on a synthetic 1D dataset with multiple changes is for our InDiD approach.}
    \label{table:multiple_metrics}
    \begin{tabular}{cccc}
        \hline
        Method & Covering $\uparrow$ & Method & Covering $\uparrow$ \\
        \hline
        BinSeg & 0.9187 & BCE & 0.9533 \\
        KernelCPD & 0.9536 & InDiD (ours) & \textbf{0.9891} \\
        \hline
    \end{tabular}
\end{table}

\section{Conclusions and limitations}

We proposed a principled loss function that allows quick change detection with a low number of false alarms for semi-structured sequential data. 
Our loss function is a differentiable lower bound for the fundamental CPD literature criteria. 
Thus, it allows end2end training of neural network models for representation learning.
Experimental evidence suggests that such an approach leads to an overall model improvement.
We observe this in various scenarios, including operating with data from multivariate sensors and video streams.

On the other hand, existing approaches have some limitations.
There is a need for a reasonably large labelled dataset to train decent models, and the usage of pre-trained general-purpose representations partially solves this problem.
We can try to apply semi-supervised methods that utilize large amounts of available unlabelled data.
Also, in applications, classic score smoothing methods can help to maintain the numerical stability of CPD models, making a combination of these two ideas a better performer.

\begin{acks}
    The research was supported by the Russian Science Foundation grant 20-71-10135. We also thank Alexander Korotin for useful discussions and critical notes during the paper writing.
\end{acks}

\bibliographystyle{ACM-Reference-Format}
\balance
\bibliography{cpd_lib}

\appendix

\appendix

\section{Appendix}

The supplementary materials contain proofs of suggested theorems in Section~\ref{sec:proofs}, results on additional experiments including details on classic CPD methods in subsection~\ref{sec:classic}, performance on unbalanced datasets in subsection~\ref{sec:balanced} and information about inference time in subsection~\ref{sec:inference_time}.
We also provide details about data preprocessing and data examples in subsection~\ref{sec:data_details}, and implementation details in subsection~\ref{sec:implementation}.

\subsection{Proofs}
\label{sec:proofs}

We provide the proof that our loss function is a lower bound of the principled Lagrange function. We start with two supplementary lemmas that lead to the main theorem.

\begin{lemma}
\label{lemma:delay_lemma}
$\tilde{L}^{\tS}_{delay}(f_\vecW, D)$ is a lower bound for the expected value of the detection delay. 
\end{lemma}

\begin{proof}
The expected value of the \emph{detection delay} $\mathcal{L}_{delay}$ given the probabilities of a change point for the scenario (B) has the following form for the change point estimate $\tau$:
\begin{equation}
\label{eq:loss_delay}
\mathcal{L}_{delay}(f_\vecW, D) = \mathbb{E}_\mathbb{\theta} (\tau - \theta)^+ = 
\frac{1}{N} \sum_{i = 1}^N \sum_{t = \theta_i}^\infty (t - \theta_i) p^i_{t} \prod_{k = \theta_i}^{t - 1} (1 - p^i_{k}),    
\end{equation}
the outer sum over $i$-s reflects that we average the loss over all sequences in a batch with change points at moments $\theta_i$; the inner sum over $t$-s corresponds to losses associated with each moment in a sequence; $p^i_t$ is the predicted probability of a change point at the moment~$t$ for $i$-th series.
The probability $p^i_{t} \prod_{k = \theta_i}^{t - 1} (1 - p^i_{k})$ reflects, that we didn't report a change point during all previous time moments, this event has the probability $\prod_{k = \theta_i}^{t - 1} (1 - p^i_{k})$, but reported a change point at time $t$ with probability $p^i_{t}$. 

The term includes an infinite number of elements. Moreover, their values decrease with each~$t$, as we multiply by quantities smaller than $1$. 
We rewrite $\mathcal{{L}}_{delay}(f_\vecW, D)$ in the following form for some~$\tS$:
\[
\begin{split}
&\mathcal{L}_{delay}(f_\vecW, D) = \frac{1}{N} \sum_{i = 1}^N \Bigg. \Bigg( \sum_{t = \theta_i}^\tS (t - \theta_i) p^i_{t} \prod_{k = \theta}^{t - 1} (1 - p^i_{k}) + \\ 
& \sum_{t = \tS + 1}^\infty (t - \theta_i) p^i_{t} \prod_{k = \theta_i}^{t - 1} (1 - p^i_{k}) \Bigg. \Bigg). 
\end{split}
\]
Consider a lower bound for the second term for a single sequence:
$ \sum_{t = \tS + 1}^\infty (t - \theta) p_t \prod_{k = \theta}^{t - 1} (1 - p_k) \geq (\tS + 1 - \theta) \sum_{t = \tS + 1}^\infty p_t \prod_{k = \theta}^{t - 1} (1 - p_k)$.
The expression under the sum introduces the probability of detecting the change moment after the moment $\tS$. 
It equals the probability of not detecting the change before moment $\tS + 1$. Thus, 
$ (\tS + 1 - \theta) \sum_{t = \tS + 1}^\infty p_t \prod_{k = \theta}^{t - 1} (1 - p_k) = (\tS + 1 - \theta) \prod_{t = \theta}^{\tS} (1 - p_t)$.
\newline
This brings us to the desired result if we sum the first part and the lower bound for the second part:
\begin{multline}
    \tilde{L}^{\tS}_{delay}(f_\vecW, D) = \frac{1}{N} \sum_{i = 1}^N \tilde{L}^{\tS}_{delay}(f_\vecW, X_i, \theta_i);\,\, \tilde{L}^{\tS}_{delay}(f_\vecW, X_i, \theta_i) = \\
    = \sum_{t = \theta_i}^{\tS} (t - \theta_i) p^i_{t} \prod_{k = \theta_i}^{t - 1} (1 - p^i_{k}) + (\tS + 1 - \theta_i) \prod_{k = \theta_i}^{\tS} (1 - p_{k}^i).
\end{multline}

\end{proof}

\begin{lemma}
\label{lemma:fa_lemma}
$\tilde{L}_{FA}(f_\vecW, D)$ is a lower bound for the expected value of the time to the false alarm. 
\end{lemma}

 The expected \emph{time to false alarm} $\mathcal{L}_{FA}$ given the probabilities for the scenario (B) has the following form:
 \begin{equation}
 \label{eq:loss_fp}
\mathcal{L}_{FA}(f_\vecW, D) = -\mathbb{E}_\theta(\tau |\  \tau<\theta) =  
-\frac{1}{N}\sum_{i = 1}^N \sum_{t = 0}^{\theta_i} t p^i_{t} \prod_{k = 0}^{t - 1} (1 - p^i_{k}).
 \end{equation}

The infinite sum here are similar to that in~\eqref{eq:loss_delay} in case $\theta_i = \infty$ --- when no change point happens. So, our approximation for $\mathcal{L}_{FA}$ is:

\begin{equation}\label{eq:loss_fp_approx}
\tilde{L}_{FA}(f_\vecW, D) = \frac{1}{N}\sum_{i = 1}^N \Bigg. \Bigg( (\tilde{T_i} + 1) \prod_{k = 0}^{\tilde{T_i}} (1 - p^i_{k}) - \sum_{t = 0}^{\tilde{T_i}} t p^i_{t} \prod_{k = 0}^{t - 1} (1 - p^i_{k}) \Bigg. \Bigg), 
\end{equation}
where $\tilde{T_i} = \min(\theta_i, T)$. 
The proof is similar to the one from \ref{lemma:delay_lemma}.

\begin{theorem}
The loss function $\tilde{L}^h(f_\vecW, D, c)$ from ~\eqref{eq:final_loss} is a lower bound for $\mathcal{\tilde{L}}(\tau)$ from criteria~\eqref{eq:criteria_lagrange}. 
\end{theorem}
\begin{proof}
The proof follows from lemmas \ref{lemma:delay_lemma} and \ref{lemma:fa_lemma} as we provide a lower bound for both terms in the loss function in them.
\end{proof}

\subsection{Additional experiments}
\label{sec:appendix_ablations}

\subsubsection{Choosing best classic methods}
\label{sec:classic} 
As "classic" CPD approaches, we consider offline methods BinSeg, PELT, and non-trainable Kernel CPD from the ruptures package~\cite{truong2020selective}. 
To provide an honest comparison, we choose hyperparameters on the training set that maximise the F1 metric and then apply them to the test set. For search, we varied the method's central models (or kernels for Kernel CPD), the penalty parameter and the number of change points to detect in every sequence. The results with the best hyperparameters are presented in Table~\ref{tab:classic_results}. We also calculate the metrics when the hyperparameters are set with respect to the maximisation of the F1 metric on a validation set to understand the maximum ability of classic approaches better. 

\begin{table*}[!ht]
    \centering
    \caption{Main quality metrics for considered classic CPD approaches. $\uparrow$ marks metrics we want to maximize, $\downarrow$ marks metrics we want to minimize.  We present the best options with hyperparameters selected using training or test data. Best values for parameters selected via using training set are highlighted with \textbf{bold} font, best values for parameters selected on test set \underline{underlined}.} 
    \label{tab:classic_results}
    \begin{tabular}{llllcccc}
    \hline
    Data for hyper- & Method & Core model & Hyperpa- & Mean Time & Mean DD $\downarrow$ & F1 $\uparrow$ & Covering $\uparrow$\\
    param. select. &  & or kernel & rameters & to FA $\uparrow$ & & & \\
    \hline
    \multicolumn{8}{c}{1D Synthetic data} \\
    \hline
    \multirow{6}{*}{ Train } 
    & BinSeg & L2 & pen = 7000 & \textbf{101.39} & 7.72 & 0.9651 & 0.9651 \\
    & BinSeg & L2 & n\_pred = 1 & 64.66 & 0.46 & 0.4772 & 0.8900 \\
    & PELT & L2 & pen = 7000 & \textbf{101.39} & 7.72 & 0.9651 & 0.9651 \\
    & PELT & L2 & n\_pred = 1 & 14.65 & \textbf{0.05} & 0.0392 & 0.7721 \\
    & Kernel CPD & Linear & pen = 21 & 94.81 & 0.64 & \textbf{0.9872} & \textbf{0.9954} \\
    & Kernel CPD & Linear & n\_pred = 1 & 62.77 & 0.12 & 0.6842 & 0.9151 \\
    \hline
    \multirow{3}{*}{ Test } 
    & BinSeg & L2 & pen = 5400 & \underline{99.48} & 5.85 & 0.6438 & 0.9679 \\
    & PELT & L2 & pen = 5600 & \underline{99.48} & 5.85 & 0.6438 & 0.9679 \\
    & Kernel CPD & Linear & pen = 16 & 94.49 & \underline{0.33} & \underline{0.9904} & \underline{0.9964} \\ 
    \hline
    \multicolumn{8}{c}{100D Synthetic data} \\
    \hline
    \multirow{6}{*}{ Train } 
    & BinSeg & L2 & pen = $2.11\cdot10^5$ & \textbf{96.94} & 3.39 & 0.6201 & 0.9788 \\
    & BinSeg & L2 & n\_pred = 1 & 62.69 & 0.34 & 0.4772 & 0.8891 \\
    & PELT & L2 & pen = $9.6\cdot10^4$ & 95.53 & 2.13 & 0.5571 & 0.9808 \\
    & PELT & L2 & n\_pred = 1 & 4.95 & 0.05 & 0.0328 & 0.8025 \\
    & Kernel CPD & Linear & pen = 0.9 & 94.20 & 0.03 & \textbf{0.9968} & \textbf{0.9996} \\
    & Kernel CPD & Linear & n\_pred = 1 & 62.59 & \textbf{0.01} & 0.6871 & 0.9053 \\
    \hline
    \multirow{3}{*}{ Test } 
    & BinSeg & L2 & pen = $10^6$ & \underline{100.19} & 6.52 & 0.6667 & 0.9683 \\
    & PELT & L2 & pen = $4.8\cdot10^4$ & 98.58 & 4.95 & 0.6778 & 0.9730 \\
    & Kernel CPD & Linear & pen = 167 & 94.20 & \underline{0.03} & \underline{0.9968} & \underline{0.9996} \\
    \hline
    \multicolumn{8}{c}{Human Activity Recognition} \\
    \hline
    \multirow{6}{*}{ Train } 
    & BinSeg & RBF & pen = 0.2 & 4.63 & 0.37 & 0.3097 & 0.6634 \\
    & BinSeg & RBF & n\_pred = 1 & \textbf{10.17} & 4.83 & 0.8563 & 0.6624 \\
    & PELT & RBF & pen = 0.5 & 4.63 & 0.37 & 0.3097 & 0.6634 \\
    & PELT & L2 & n\_pred = 1 & 9.86 & 0.92 & 0.5740 & 0.8414 \\
    & Kernel CPD & RBF & pen = 0.3 & 5.90 & \textbf{0.07} & 0.6703 & \textbf{0.8538} \\
    & Kernel CPD & Cosine & n\_pred = 1 & 9.35 & 1.26 & \textbf{0.9023} & 0.8510 \\
    \hline
    \multirow{3}{*}{ Test } 
    & BinSeg & RBF & pen = 0.5 & \underline{10.17} & 4.83 & 0.8563 & 0.6624 \\
    & PELT & RBF & pen = 0.6 & \underline{10.17} & 4.83 & 0.8563 & 0.6624 \\
    & Kernel CPD & Linear & pen = 0.9 & 9.57 & \underline{0.09} & \underline{0.8761} & \underline{0.9450} \\
    \hline
    \multicolumn{8}{c}{Sequences of MNIST images} \\
    \hline
    \multirow{6}{*}{ Train } 
    & BinSeg & L2 & pen = 402 & \textbf{45.26} & 5.08 & 0.4901 & \textbf{0.8430} \\
    & BinSeg & L2 & n\_pred = 1 & 29.16 & 1.62 & 0.4252 & 0.7029 \\
    & PELT & L2 & pen = 402 & 45.26 & 5.08 & 0.4901 & \textbf{0.8430} \\
    & PELT & RBF & n\_pred = 1 & 9.95 & \textbf{0.05} & 0.0583 & 0.7259 \\
    & Kernel CPD & Linear & pen = 375 & 43.02 & 4.08 & \textbf{0.5500} & 0.8364 \\
    & Kernel CPD & RBF & n\_pred = 1 & 28.88 & 3.81 & 0.4887 & 0.6420 \\
    \hline
    \multirow{3}{*}{ Test } 
    & BinSeg & L2 & pen = 439 & \underline{47.99} & 5.63 & 0.5172 & \underline{0.8697} \\
    & PELT & L2 & pen = 439 & \underline{47.99} & 5.63 & 0.5172 & \underline{0.8697} \\
    & Kernel CPD & Linear & pen = 357 & 40.94 & \underline{3.20} & \underline{0.5724} & 0.8270 \\
    \hline
    \end{tabular} 
\end{table*}

In our opinion, setting the number of change points to detect in every sequence (n\_pred) to 1 is an unsuccessful choice as it provides a change point in every sequence, and the number of false alarms rapidly increases, so we avoid using it. One of the apparent disadvantages of classic methods is the necessity to pick hyperparameters manually, while representation-based methods do it during the learning process. We can see from Table~\ref{tab:classic_results} that the most suitable choices for the training set do not provide the best metrics on the test set in most cases. 

\subsubsection{Experiment on imbalanced datasets}
\label{sec:balanced}

In our main experiments, we mostly use balanced training datasets.
They include an equal number of sequences with and without change points. 
However, we expect to face change points less often than normal sequences without them in real-world data. 
Thus, to investigate the performance of proposed approaches in more realistic scenarios, we conduct additional experiments with sequences of MNIST images. 
We try a different number of sequences with change points and calculate the main quality metrics.
We test our three main methods based on representation learning (InDiD, BCE, BCE+InDiD) with the same model design as for the main experiments (see ~\ref{sec:experiments}, ~\ref{sec:implementation}).
The results are given in Table~\ref{table:imbalance_metrics}. 

\begin{table}[!ht]
    \centering
    \caption{Quality metrics for models trained on the datasets with a smaller amount of sequences with change points ($\#$ of seq. with CP), while we use $350$ sequences without change points for all the experiments. $\uparrow$ marks metrics we want to maximize, $\downarrow$ marks metrics we want to minimize. The results are averaged by 5 runs and are given in the format $mean\pm std$. Best values are highlighted with \textbf{bold} font.}
    \begin{tabular}{llcc}
    \hline
        $\#$ of seq. & Method & AUC $\downarrow$ & Max Covering $\uparrow$ \\ 
         with change & ~ & ~ & ~ \\ 
         point & ~ & ~ & ~ \\ \hline
        350 & InDiD & $213.76 \pm 10.77$ & $0.9656 \pm 0.0012$ \\
        (original) & BCE  & $237.94 \pm 8.54$ & \bm{$0.9661 \pm 0.0006$} \\
        ~ & BCE+InDiD & \bm{$202.19 \pm 2.07$} & $0.9656 \pm 0.0012$ \\ \hline
        250  & InDiD & $224.73 \pm 4.80$ & \bm{$0.9671 \pm 0.0029$} \\
        ~ & BCE  & $241.31 \pm 2.12$ & $0.9654 \pm 0.0038$ \\
        ~ & BCE+InDiD & \bm{$224.52 \pm 17.75$} & $0.9665 \pm 0.0031$ \\ \hline
        150  & InDiD & \bm{$230.67 \pm 5.57$} & $0.9644 \pm 0.0034$ \\
        ~ & BCE  & $239.384 \pm 3.802$ & \bm{$0.9657 \pm 0.0039$} \\
        ~ & BCE+InDiD & $237.29 \pm 16.59$ & $0.9644 \pm 0.0038$  \\ \hline
        50  & InDiD & \bm{$237.78 \pm 7.96$} & $0.9618 \pm 0.0058$  \\
        ~ & BCE  & $245.96 \pm 7.79$ & $0.9616 \pm 0.0046$ \\
        ~ & BCE+InDiD & $247.35 \pm 7.78$ & \bm{$0.9621 \pm 0.0046$} \\ \hline
        25  & InDiD & \bm{$239.82 \pm 11.63$} & \bm{$0.9565 \pm 0.0024$} \\
        ~ & BCE  & $248.54 \pm 5.65$ & $0.9563 \pm 0.0039$  \\
        ~ & BCE+InDiD & $253.21 \pm 13.44$ & $0.9562 \pm 0.0023$ \\ \hline
    \end{tabular}
    \label{table:imbalance_metrics}
\end{table}

We see that the performances of all the considered models deteriorate as the amount of "abnormal" data available for training decreases: Area under the Detection Curve increases, Covering decreases.
However, the models can still detect change points quite precisely, even when the share of sequences with change points in the training dataset is less than $7\%$.
We also see that our InDiD and BCE+InDiD methods are better than baseline BCE for a different share of "abnormal" data in the train set in terms of the AUC metric.

\subsubsection{Computational time comparison}
\label{sec:inference_time}

Inference time is a crucial parameter for a model's industrial application. We share the running time of considered approaches for two different datasets in Table~\ref{tab:time}.
The presented time is the mean value over $10$ runs.
KL-CPD and TSCP suffer from the necessity to compare several subsequences to obtain the predictions for the whole data stream in an online regime. So, their inference time is higher than those of other representation-based methods. 

    \begin{table}[!ht]
        \centering
        \caption{CPU inference time for one sequence for considered methods.}
        \label{tab:time}
        \begin{tabular}{lcc}
        \hline
        & Synthetic 100D & MNIST sequence \\
        \hline
        Method & \multicolumn{2}{c}{Inference time, ms} \\
        \hline
        BinSeg & 6.550 $\pm$ 0.018 & 5.850 $\pm$ 0.004 \\
        PELT & 8.450 $\pm$ 0.009 & 4.000 $\pm$ 0.006 \\
        Kernel\_CPD & 4.340 $\pm$ 0.003 & 4.750 $\pm$ 0.002 \\
        KL-CPD & $(1.700 \pm 0.008)\cdot 10^{3}$ & $(6.520 \pm 0.031)\cdot 10^{3}$ \\
        TSCP &  $(2.940 \pm 0.039)\cdot 10^{3}$ & $(6.280 \pm 0.037)\cdot 10^{3}$ \\
        BCE & 5.130 $\pm$ 0.016 & 2.660 $\pm$ 0.013 \\
        InDiD & 5.070 $\pm$ 0.021 & 2.650 $\pm$ 0.018 \\
        BCE+InDiD & 5.130 $\pm$ 0.021 & 2.650 $\pm$ 0.016 \\
        \hline
        \end{tabular} 
    \end{table}

\begin{table*}[!ht]
\caption{Statistics for the datasets used in our experiments.} \label{tab:datasets}
\centering
\begin{tabular}{@{}llcccccc@{}}
\toprule
Dataset   & Single sample & Train & Test & $\%$ of sequences with & $\%$ of sequences with & Sequence \\ 
  & shape & size & size & changes in train set & changes in test set & length \\ 
\midrule
1D Synthetic data & $1 \times 1$ & $700$ & $300$ & $48.9$ & $52.7$ & $128$\\
100D Synthetic data & $100 \times 1$ & $700$ & $300$ & $48.9$ & $52.7$ & $128$ \\
Human Activity Recognition & $28 \times 1$ & $3580$ & $1337$ & $83.9$ & $87.1$ & $20$ \\
Sequences of MNIST images & $28 \times 28 \times1$     & $700$  & $300$ & $51.1$ & $47.3$      & $64$\\
Explosions & $ 256 \times 256 \times 3$   & $310$ & $315$ & $50.0$ & $4.8$        & $16$\\
Road Accidents & $ 256 \times 256 \times 3$   & $666$ & $349$ & $50.0$ & $14.0$        & $16$\\
\toprule
\end{tabular}
\end{table*}

\begin{figure*}[!ht]
\begin{subfigure}{.5\textwidth}
 \centering
 \includegraphics[width=0.75\linewidth]{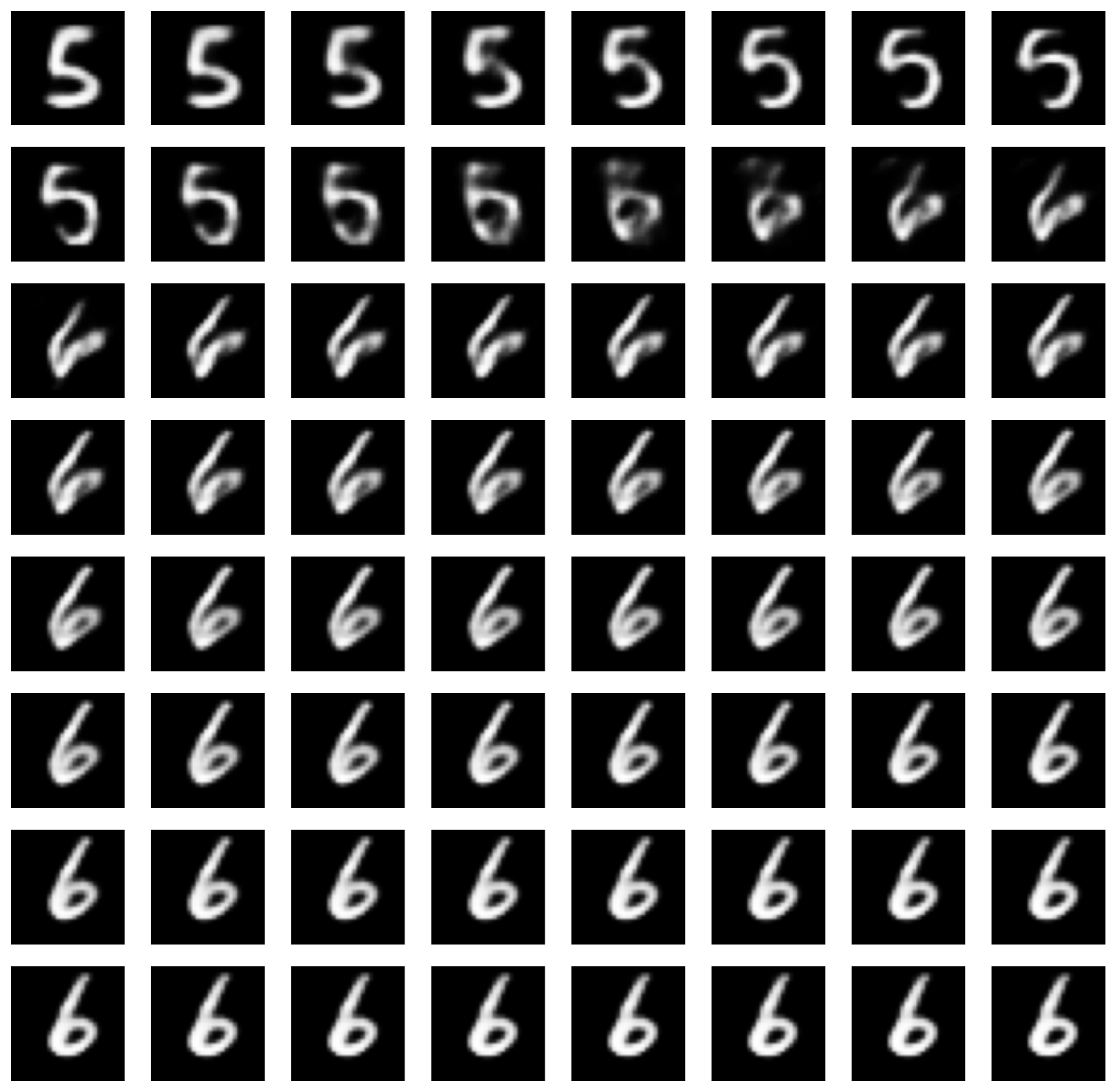}
\end{subfigure}%
\begin{subfigure}{.5\textwidth}
 \centering
 \includegraphics[width=0.75\linewidth]{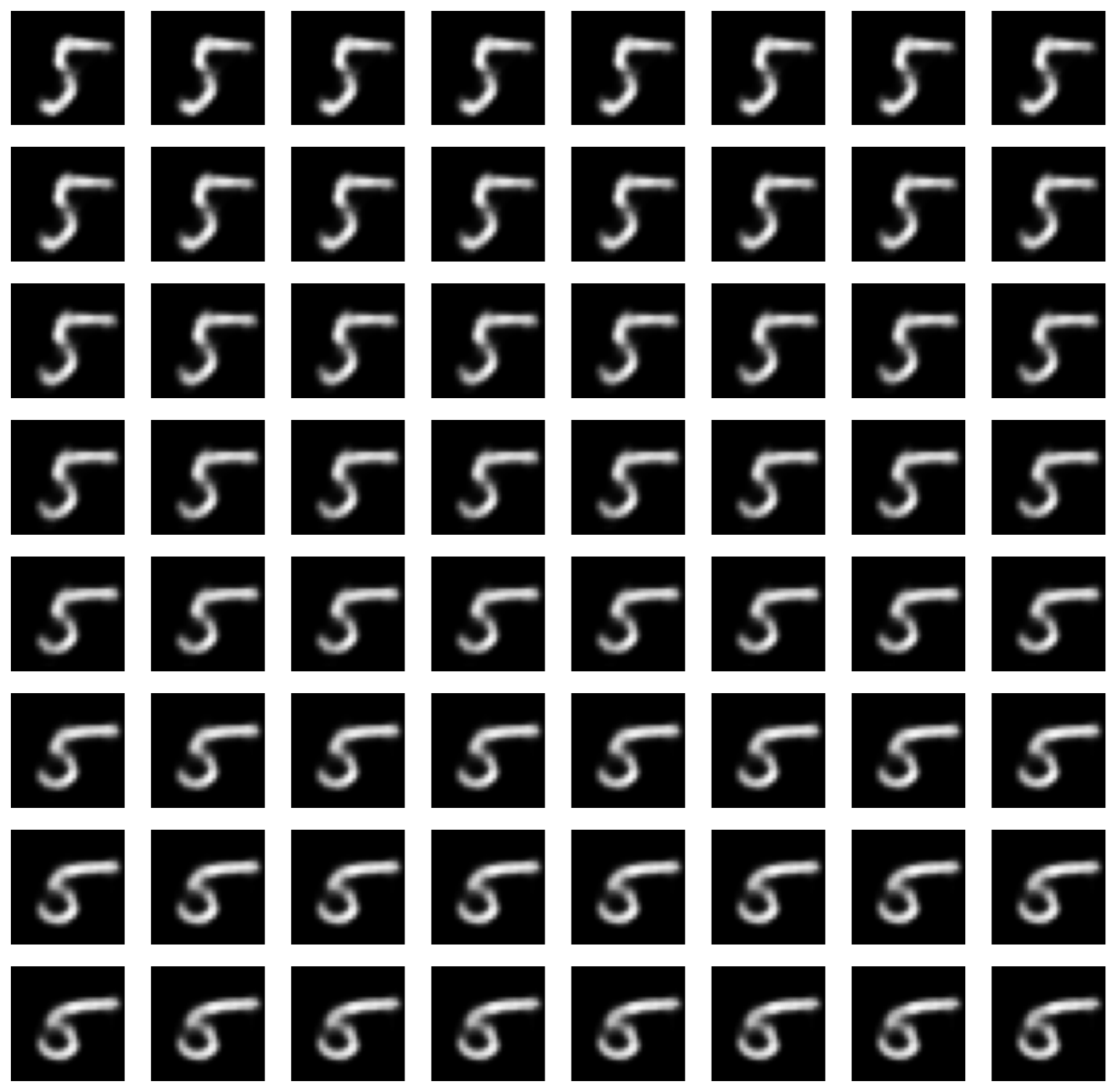}
\end{subfigure}
\caption{Examples of MNIST sequences with (left) and without (right) a change point. At each row, a sequence goes from left to right. Then we continue with the leftmost image in the next row.} \label{fig:mnist_examples}.
\end{figure*}

\begin{figure*}[!ht]
\begin{subfigure}{.5\textwidth}
 \centering
 \includegraphics[width=0.75\linewidth]{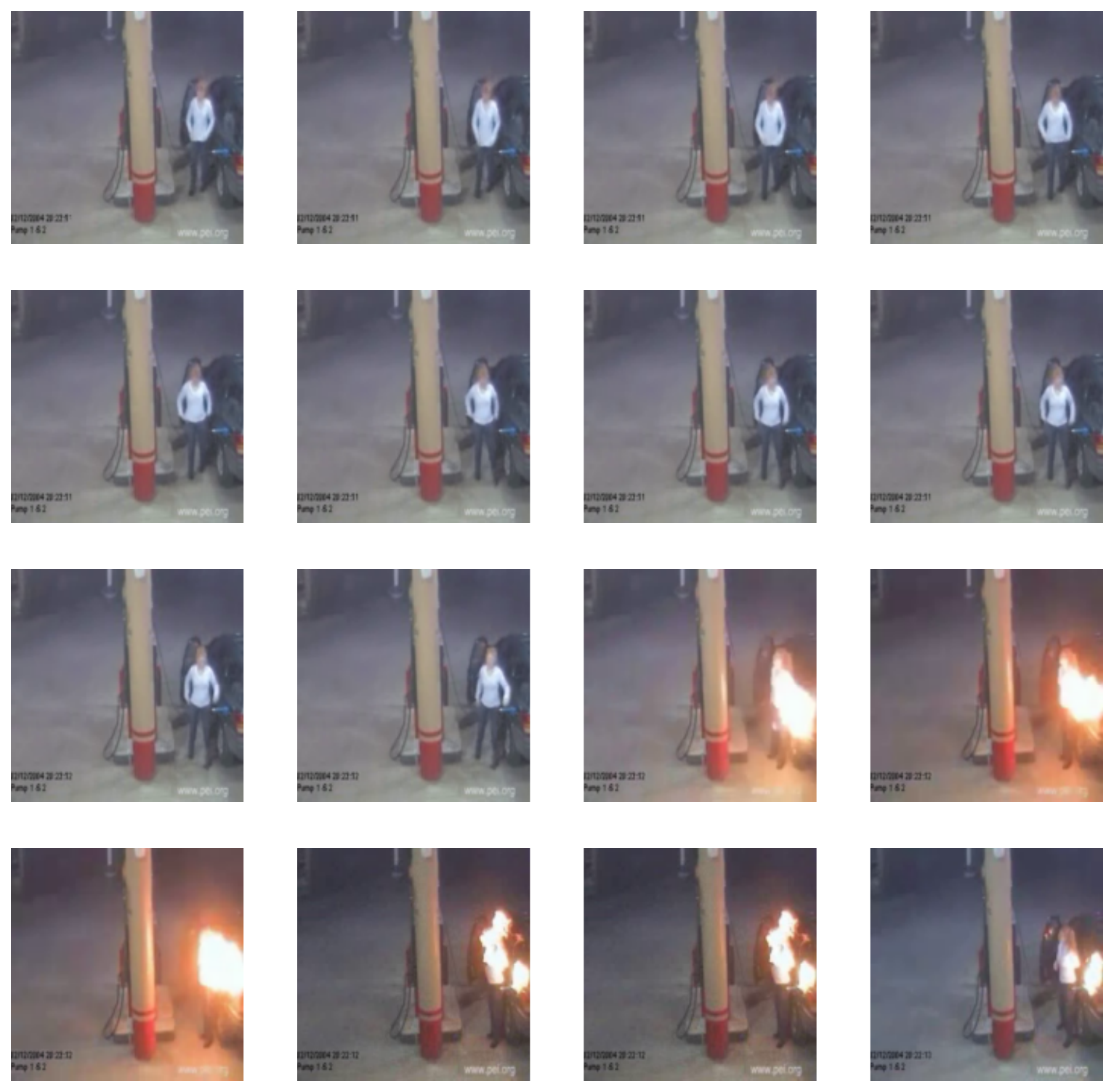}
\end{subfigure}%
\begin{subfigure}{.5\textwidth}
 \centering
 \includegraphics[width=0.75\linewidth]{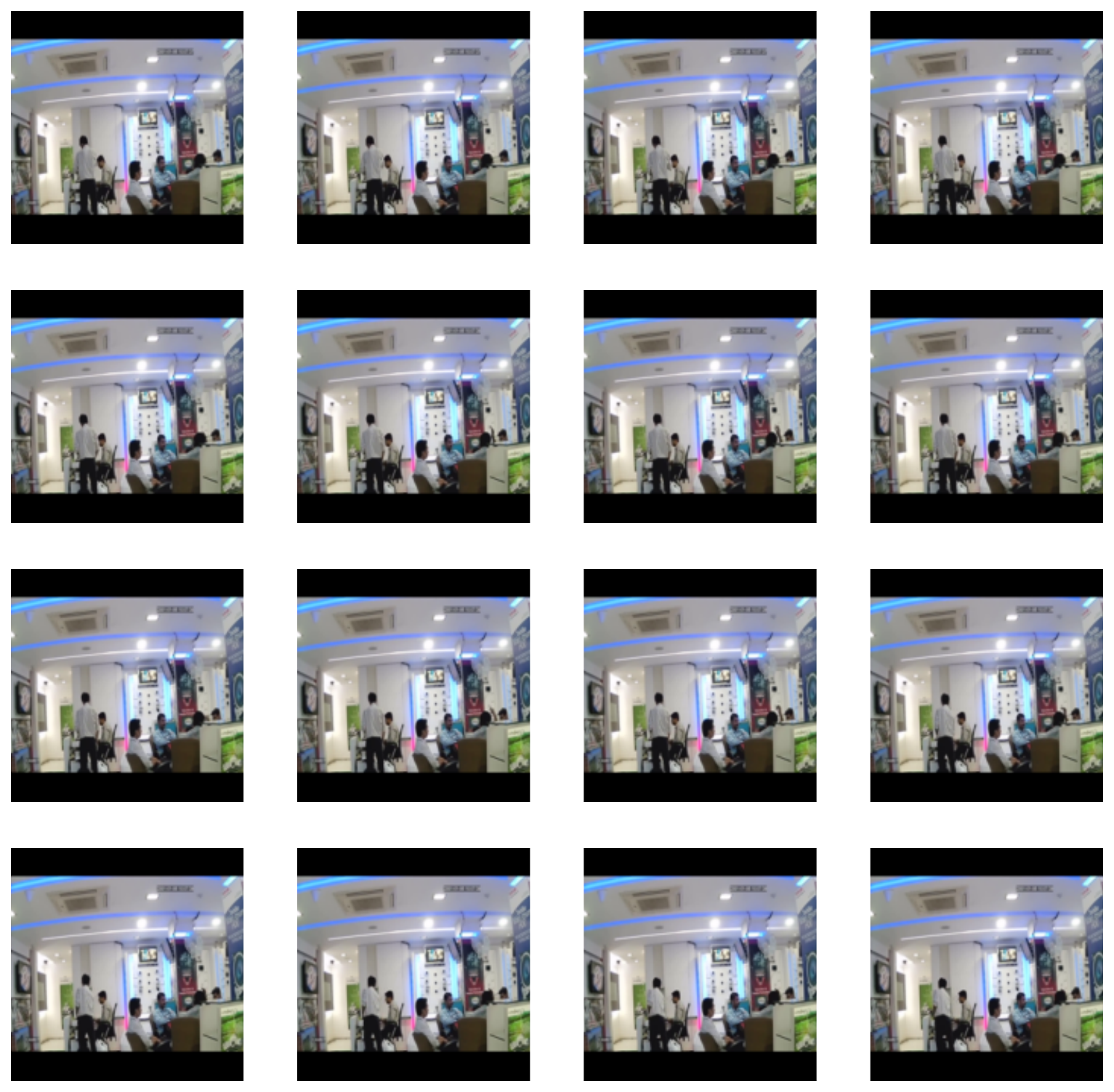}
\end{subfigure}
\caption{Examples from Explosion data with (left) and without (right) a change point. At each row, a sequence goes from left to right. Then we continue with the leftmost image in the next row.} \label{fig:explosion_examples}.
\end{figure*}

\subsection{Datasets details}
\label{sec:data_details}

In the paper, we consider a diverse set of datasets, with two of them introduced for the first time for the problem of supervised change point detection. Summary statistics on the used data are shown in Table \ref{tab:datasets}.
More information about data generation and preprocessing one can find below. For simplicity, we call sequences with change points "abnormal" and sequences without change points "normal". 

\textbf{Synthetic datasets.}  
To generate an abnormal sequence, we use two subsequences with randomly selected lengths for every dataset sample. The first subsequence corresponds to the normal state and is sampled from $\mathcal{N}(1, 1)$ in the 1D case or from $\mathcal{N}(\mathbf{1}, I)$ in the 100D case. The second one, abnormal, is from a Gaussian with a randomly selected mean between $2$ and $100$ (constant vector for the multivariate scenario) and the variance $1$ or $I$. The sequence without a change point is generated similarly to the first subsequence. We use balanced datasets with an equal number of normal and abnormal objects. 

\textbf{Human Activity dataset} contains sensors' records for $12$ types of human physical activity, such as walking, elevator, going upstairs, running, sleeping, etc., collected from the group of $30$ individuals. Each record has a length of $200$ and consists of $561$ features. The frequency of sensor measurements is $50$ Hz. 
We cut these records to get subsequences with changes in the type of human activity and without them. The length of every subsequence is $20$. We also select only $28$ features out of $561$ corresponding to the "mean" parameters. It is important to note that the activity type frequently changes in this dataset, so the number of normal examples is scarce, even for a small length of $20$.

\textbf{Sequences of MNIST images.}
To generate this dataset, we obtain latent representations for MNIST images using a pretty simple Conditional Variational Autoencoder \cite{sohn2015learning} (CVAE). The utilized architecture has a hidden dimension of $256$ and a latent dimension of $75$. 
Then we take two points in the latent space corresponding to a pair of digits and add $62$ points from the line connecting them. As a result, we have a sequence with a length of $64$ of latent representations of images. After reconstruction via CVAE's decoder, we get a sequence of images. Such an approach allows us to generate images with smooth transitions of digits, even if the first and the last digits differ. We expect these gradual changes to increase the difficulty of CPD. There are sequences with (e.g. from 4 to 7) and without (e.g. from 4 to 4) a change point in the result dataset. We generate 1000 sequences of length 64 distributed evenly between normal and abnormal samples. Two examples of obtained sequences are in Figure~\ref{fig:mnist_examples}.

\textbf{Explosions and Road Accidents.}
The whole dataset \cite{sultani2018real} consists of real-world $240\times320$ RGB videos with 13 realistic anomaly types such as explosion, road accident, burglary, etc., and normal examples. As we have already mentioned, we use only explosion and road accident videos for our research. The CPD specific requires a change in data distribution. We suppose that explosions and road accidents correspond to such a scenario, while most other types correspond to point anomalies. For example, data, obviously, comes from a normal regime before the explosion. After it, we can see fire and smoke, which last for some time. Thus, the first moment when an explosion appears is a change point.
Along with a volunteer, the authors carefully labelled chosen anomaly types. Their opinions were averaged. We provide the obtained markup and the code and release it for the first time, so other researchers can use it to validate their CPD algorithms\footnote{The data labelling are available at \url{ https://github.com/romanenkova95/InDiD}}.

To generate the dataset, we apply methods from torchvision \cite{paszke2019pytorch} that cut video on small clips (a subsequence of the initial frames' sequences) with a given step and frequency. We use $30$ frames per second, and the number of frames in the clip is $16$. For training data, we sample clips with step $5$, so there are overlapping sections. For test data, we cut the video on non-overlapping clips. We also transform videos into the shape $256\times 256$ and normalize them. The examples from the obtained data are shown in Figure \ref{fig:explosion_examples}.

\subsection{Implementation details}
\label{sec:implementation}

The main implementation details of all the considered CPD methods are summarized in Table~\ref{tab:implementation}. We denote window size for KL-CPD and TSCP methods (see ~\cite{chang2018kernel, deldari2021tscp2} for details) as "w.s.", the batch size used for training as "b.s.", learning rate as "l.r." and monitoring parameter for early stopping as "mon". 

For video preprocessing in our Explosions and Road Accidents experiments, we use a pre-trained feature extractor X3D \cite{feichtenhofer2019slowfast} that is denoted as "Extractor" in Table~\ref{tab:implementation}. We freeze its parameters while training all the CPD models. 

We mostly follow the original procedures for KL-CPD\footnote{\url{ https://github.com/OctoberChang/klcpd_code}} and TSCP\footnote{\url{ https://github.com/cruiseresearchgroup/TSCP2}} presented in papers and official repositories \cite{chang2018kernel, deldari2021tscp2}, varying only their hyperparameters. 
During the application of these methods, we faced several difficulties described below. To treat them, we had to implement some procedures out of the initial research scope. We expect deeper investigations to improve the performance, while they should be addressed in separate papers. In particular, for KL-CPD and TSCP, we highlight:
\begin{itemize} 
\item KL-CPD method processes a sequence of observations and outputs scores based on MMD statistic that corresponds to the similarity of the distributions of the two consecutive windows. To adapt this method to our setting, which works with change point probabilities, we propose transforming these non-negative MMD scores into predictions in the following way: $pred = \tanh(a\cdot scores)$. Thus, we scale the outputs to $[0, 1]$ and can interpret them as change probabilities similar to our InDiD method. Constant $a$ is a hyperparameter. We chose the following coefficients for our experiments: $a=100$ -- for Synthetic 1D, 100D and MNIST sequences; $a = 5$ -- for Human Activity Recognition; $a = 10^{7}$ -- for Explosions and Road Accidents. In a similar way, TSCP outputs the cosine similarity of consecutive windows. So, we use sigmoid function to scale the output, i.e. $pred = \sigma(-scores)$.  
\item As the KL-CPD approach is based on several recurrent networks that encode and decode the original sequence, it requires many parameters in multidimensional scenarios to decode the embeddings back into the original space. Direct applying it to video data turned out to be impossible due to the huge resource necessities. We overcome this issue by adding Linear layers to encoders and decoders to make the embedding size small enough to be stored and processed.
\item Another challenge was connected with training TSCP on synthetically generated data, such as 1D, 100D and MNIST. Because of our data generation procedure, there are many pretty similar sequences, as we always start from $f_\infty$ and end with $f_0$. TSCP tries by default to separate these close sequences in the embedding space. Thus, we suppose that such peculiarities negatively affect the method's performance and to make it work in such scenarios, we should change the sampling procedure and the training process in the method.
\end{itemize}

\begin{table*}[!ht]
    \centering
    \caption{Implementation details for the considered CPD methods.} \label{tab:implementation}
    \begin{tabular}{lllll}
    \hline
        Method & Core's & Loss' & \hspace{1.0mm}Learning & \hspace{1.0mm}Early Stopping \\
        ~ & architecture & parameters & \hspace{1.0mm}parameters & \hspace{1.0mm}parameters  \\ \hline
        \multicolumn{5}{c}{\textbf{1D Synthetic data}} \\ \hline
        Best classic & \hspace{1.5mm}Kernel CPD (kernel = linear) & pen = 21 & \hspace{1.0mm} na & \hspace{1.5mm}na  \\ \hline
        \hspace{-2.5mm} \begin{tabular}{l}
        KL-CPD \\
        (w.s. = 8) \\
        \end{tabular}
        & 
        \begin{tabular}{lll}
        Net D/G & Enc & GRU(1, 4) \\
        Net D & Dec & GRU(4, 1) \\
        Net G & Dec & GRU(1, 4) + Lin(4, 1) \\
        \end{tabular} &
        \hspace{-2.5mm}
        \begin{tabular}{l}
        lambda\_ae = 0.1 \\
        lambda\_real = 10 \\
        med\_sqdist = 1                
        \end{tabular} & 
        \begin{tabular}{l}
        b.s = 64 \\
        l.r. = $10^{-3}$ \\               
        \end{tabular} &
        \begin{tabular}{l}
        mon = val\_mmd2\_real \\
        delta = $10^{-5}$ \\
        patience = 5 \\               
        \end{tabular}
        \\ \hline
        \hspace{-2.5mm} \begin{tabular}{l}
        TSCP \\
        (w.s. = 16) \\
        \end{tabular}
        & 
        \begin{tabular}{ll}
        TSCP & (nb\_filters = 16, n\_steps = 4 \\
        &  code\_size = 4, kernel\_size = 5) \\
        \end{tabular}
        & temp\_NCE = 0.5 
        & \begin{tabular}{l}
        b.s = 4 \\
        l.r. = $10^{-4}$ \\               
        \end{tabular} & 
        \begin{tabular}{l}
        mon = train loss \\
        delta = $10^{-4}$ \\
        patience = 5 \\               
        \end{tabular}
        \\ \hline
        \hspace{-2.5mm} \begin{tabular}{l}
        BCE \\
        InDiD \\
        BCE + InDiD \\
        \end{tabular}
        & \begin{tabular}{l}
        LSTM(1, 4, dropout = 0.5) + \\
        + Linear(4, 1) + Sigmoid() \\
        \end{tabular}
        & $T=32$ &
        \begin{tabular}{l}
        b.s = 64 \\
        l.r. = $10^{-3}$ \\               
        \end{tabular} & 
        \begin{tabular}{l}
        mon = val loss \\
        delta = 0 \\
        patience = 10 \\
        \end{tabular} \\\hline
        \multicolumn{5}{c}{\textbf{100D Synthetic data}} \\ \hline
        Best classic & \hspace{1.5mm}Kernel CPD (kernel = rbf) & pen = 0.9 & \hspace{1.0mm} na & \hspace{1.5mm}na \\ \hline
        \hspace{-2.5mm} \begin{tabular}{l}
        KL-CPD \\
        (w.s. = 8) \\
        \end{tabular}
        & 
        \begin{tabular}{lll}
        Net D/G & Enc & GRU(100, 16) \\
        Net D & Dec & GRU(16, 100) \\
        Net G & Dec & GRU(100, 16) + Lin(16, 100) \\
        \end{tabular} &
        \hspace{-2.5mm}        
        \begin{tabular}{l}
        lambda\_ae = 0.1 \\
        lambda\_real = 10 \\
        med\_sqdist = 3                
        \end{tabular}     
        & 
        \begin{tabular}{l}
        b.s = 64 \\
        l.r. = $10^{-3}$ \\               
        \end{tabular} & 
        \begin{tabular}{l}
        mon = val\_mmd2\_real \\
        delta = $10^{-5}$ \\
        patience = 1 \\
        \end{tabular} \\\hline
        \hspace{-2.5mm} \begin{tabular}{l}
        TSCP \\
        (w.s. = 16) \\
        \end{tabular}
        &
        \begin{tabular}{ll}
        TSCP & (nb\_filters = 4, n\_steps = 4 \\
        &  code\_size = 4, kernel\_size = 5) \\
        \end{tabular}
        & temp\_NCE = 0.5 & 
        \begin{tabular}{l}
        b.s = 4 \\
        l.r. = $10^{-4}$ \\               
        \end{tabular} &
        \begin{tabular}{l}
        mon = train loss \\
        delta = $10^{-4}$ \\
        patience = 5 \\
        \end{tabular}  
        \\ \hline
        \hspace{-2.5mm} \begin{tabular}{l}
        BCE \\
        InDiD \\
        BCE + InDiD \\
        \end{tabular}
        & 
        \begin{tabular}{l}
        LSTM(100, 8, dropout = 0.5) + \\
        + Linear(8, 1) + Sigmoid() \\
        \end{tabular}
        & T $= 32$ & 
        \begin{tabular}{l}
        b.s = 64 \\
        l.r. = $10^{-3}$ \\               
        \end{tabular} & 
        \begin{tabular}{l}
        mon = val loss \\
        delta = 0 \\
        patience = 10 \\
        \end{tabular} \\\hline
        \multicolumn{5}{c}{\textbf{Human Activity Recognition}} \\ \hline
        Best classic & \hspace{1.5mm}Kernel CPD (kernel = rbf) & pen = 0.3 & \hspace{1.0mm} na & \hspace{1.5mm}na \\ \hline
        \hspace{-2.5mm} \begin{tabular}{l}
        KL-CPD \\
        (w.s. = 3) \\
        \end{tabular}
        & 
        \begin{tabular}{lll}
        Net D/G & Enc & GRU(28, 8) \\
        Net D & Dec & GRU(8, 28) \\
        Net G & Dec & GRU(28, 8) + Lin(8, 28) \\
        \end{tabular} &
        \hspace{-2.5mm}        
        \begin{tabular}{l}
        lambda\_ae = 0.2 \\
        lambda\_real = 1 \\
        med\_sqdist = 2                
        \end{tabular}     
        & 
        \begin{tabular}{l}
        b.s = 64 \\
        l.r. = $10^{-4}$ \\               
        \end{tabular} & 
        \begin{tabular}{l}
        mon = val\_mmd2\_real \\
        delta = $10^{-5}$ \\
        patience = 5 \\
        \end{tabular} \\\hline
        \hspace{-2.5mm} \begin{tabular}{l}
        TSCP \\
        (w.s. = 4) \\
        \end{tabular}
        &
        \begin{tabular}{ll}
        TSCP & (nb\_filters = 64, n\_steps = 8 \\
        &  code\_size = 4, kernel\_size = 5) \\
        \end{tabular}
        & temp\_NCE = 0.1 & 
        \begin{tabular}{l}
        b.s = 8 \\
        l.r. = $10^{-3}$ \\               
        \end{tabular} &
        \begin{tabular}{l}
        mon = train loss \\
        delta = $10^{-4}$ \\
        patience = 5 \\
        \end{tabular}  
        \\ \hline
        \hspace{-2.5mm} \begin{tabular}{l}
        BCE \\
        InDiD \\
        BCE + InDiD \\
        \end{tabular}
        & 
        \begin{tabular}{l}
        LSTM(28, 8, dropout = 0.5) + \\
        + Linear(8, 1) + Sigmoid() \\
        \end{tabular}
        & T $= 5$ & 
        \begin{tabular}{l}
        b.s = 64 \\
        l.r. = $10^{-3}$ \\               
        \end{tabular} & 
        \begin{tabular}{l}
        mon = val loss \\
        delta = 0 \\
        patience = 10 \\
        \end{tabular} \\\hline
        \multicolumn{5}{c}{\textbf{Sequences of MNIST images}} \\ \hline
        Best classic & \hspace{1.5mm}Kernel CPD (kernel = linear) & pen = 375 & \hspace{1.0mm} na & \hspace{1.5mm}na \\ \hline
        \hspace{-2.5mm} \begin{tabular}{l}
        KL-CPD \\
        (w.s. = 8) \\
        \end{tabular}
        & 
        \begin{tabular}{lll}
        Net D/G & Enc & GRU(784, 32) \\
        Net D & Dec & GRU(32, 784) \\
        Net G & Dec & GRU(784, 32) + Lin(32, 784) \\
        \end{tabular} &
        \hspace{-2.5mm}        
        \begin{tabular}{l}
        lambda\_ae = 0.001 \\
        lambda\_real = 0.01 \\
        med\_sqdist = 10               
        \end{tabular}     
        & 
        \begin{tabular}{l}
        b.s = 64 \\
        l.r. = $3\cdot 10^{-4}$ \\               
        \end{tabular} & 
        \begin{tabular}{l}
        mon = val\_mmd2\_real \\
        delta = $10^{-5}$ \\
        patience = 5 \\
        \end{tabular} \\\hline
        \hspace{-2.5mm} \begin{tabular}{l}
        TSCP \\
        (w.s. = 16) \\
        \end{tabular}
        &
        \begin{tabular}{ll}
        TSCP & (nb\_filters = 256, n\_steps = 128 \\
        &  code\_size = 8, kernel\_size = 5) \\
        \end{tabular}
        & temp\_NCE = 0.1 & 
        \begin{tabular}{l}
        b.s = 64 \\
        l.r. = $10^{-3}$ \\               
        \end{tabular} &
        \begin{tabular}{l}
        mon = train loss \\
        delta = $10^{-4}$ \\
        patience = 5 \\
        \end{tabular}  
        \\ \hline
        \hspace{-2.5mm} \begin{tabular}{l}
        BCE \\
        InDiD \\
        BCE + InDiD \\
        \end{tabular}
        & 
        \begin{tabular}{l}
        LSTM(784, 32, dropout = 0.25) + \\
        + Linear(32, 1) + Sigmoid() \\
        \end{tabular}
        & T $= 32$ & 
        \begin{tabular}{l}
        b.s = 64 \\
        l.r. = $10^{-3}$ \\               
        \end{tabular} & 
        \begin{tabular}{l}
        mon = val loss \\
        delta = 0 \\
        patience = 10 \\
        \end{tabular} \\\hline
        \multicolumn{5}{c}{\textbf{Explosions / Road Accidents}} \\ \hline
        \hspace{-2.5mm} \begin{tabular}{l}
        KL-CPD \\
        (w.s. = 8) \\
        \end{tabular}
        & 
        \begin{tabular}{lll}
        Net D/G & Enc & Extractor + Lin(12288, 100) + \\
        & ~ & + ReLU() + GRU(100, 16) \\
        Net D & Dec & GRU(16, 100) + \\
        & ~ & + Lin(100, 12288) + ReLU() \\
        Net G & Dec & GRU(100, 16) + \\
        ~ & ~ & + Lin(16, 12288) \\
        \end{tabular} &
        \hspace{-2.5mm}        
        \begin{tabular}{l}
        lambda\_ae = 0.1 \\
        lambda\_real = 10 \\
        med\_sqdist = 50               
        \end{tabular}     
        & 
        \begin{tabular}{l}
        b.s = 8 \\
        l.r. = $10^{-4}$ \\               
        \end{tabular} & 
        \begin{tabular}{l}
        mon = val\_mmd2\_real \\
        delta = $10^{-5}$ \\
        patience = 5 \\
        \end{tabular} \\\hline
        \hspace{-2.5mm} \begin{tabular}{l}
        TSCP \\
        (w.s. = 4) \\
        \end{tabular}
        &
        \begin{tabular}{ll}
        \multicolumn{2}{l}{Extractor +} \\
        + TSCP & (nb\_filters = 256, n\_steps = 256 \\
        &  code\_size = 128, kernel\_size = 5) \\
        \end{tabular}
        & temp\_NCE = 0.1 & 
        \begin{tabular}{l}
        b.s = 16 \\
        l.r. = $10^{-3}$ \\               
        \end{tabular} &
        \begin{tabular}{l}
        mon = train loss \\
        delta = $10^{-4}$ \\
        patience = 5 \\
        \end{tabular}  
        \\ \hline
        \hspace{-2.5mm} \begin{tabular}{l}
        BCE \\
        InDiD \\
        BCE + InDiD \\
        \end{tabular}
        & 
        \begin{tabular}{l}
        Extractor + \\
        + LSTM(12288, 64, dropout = 0.5) + \\
        + Linear(64, 1) + Sigmoid() \\
        \end{tabular}
        & T $= 8$ & 
        \begin{tabular}{l}
        b.s = 16 \\
        l.r. = $10^{-3}$ \\               
        \end{tabular} & 
        \begin{tabular}{l}
        mon = val loss \\
        delta = 0 \\
        patience = 10 \\
        \end{tabular} \\\hline
    \end{tabular}
\end{table*}

\end{document}